\documentclass[letterpaper]{article} 
\usepackage{aaai24}  
\usepackage{times}  
\usepackage{helvet}  
\usepackage{courier}  
\usepackage[hyphens]{url}  
\usepackage{graphicx} 
\usepackage{natbib}  
\usepackage{caption} 
\usepackage{algorithm}
\usepackage{algorithmic}

\usepackage{newfloat}
\usepackage{listings}
\DeclareCaptionStyle{ruled}{labelfont=normalfont,labelsep=colon,strut=off} 
\floatstyle{ruled}
\newfloat{listing}{tb}{lst}{}
\floatname{listing}{Listing}
\usepackage{enumitem}
\usepackage{multirow}

\usepackage{booktabs}

\usepackage{amsmath,amsfonts,amssymb,amsthm}

\newcommand{\Expect}{\mathbb{E{}}}

\newcommand{\dist}{\mathop{\bf dist{}}}
\newcommand{\argmin}{\mathop{\rm argmin}}

\newcommand{\ie}{{\it i.e.}}

\newtheorem{thm}{Theorem}[section]
\newtheorem{lem}{Lemma}[section]

\newtheorem{asmp}{Assumption}[section]
\newtheorem{defn}{Definition}[section]

\newtheorem{rem}{Remark}[section]

\def\approxcorrect{\checkmark\kern-1.1ex\raisebox{.89ex}{$\times$}}

\usepackage{amsmath,amsfonts,bm}

\def\eqref#1{equation~\ref{#1}}

\def\ceil#1{\lceil #1 \rceil}
\def\floor#1{\lfloor #1 \rfloor}
\def\1{\bm{1}}

\DeclareMathAlphabet{\mathsfit}{\encodingdefault}{\sfdefault}{m}{sl}
\SetMathAlphabet{\mathsfit}{bold}{\encodingdefault}{\sfdefault}{bx}{n}

\def\gA{{\mathcal{A}}}
\def\gB{{\mathcal{B}}}
\def\gC{{\mathcal{C}}}
\def\gD{{\mathcal{D}}}
\def\gE{{\mathcal{E}}}
\def\gF{{\mathcal{F}}}

\def\gH{{\mathcal{H}}}

\def\gL{{\mathcal{L}}}
\def\gM{{\mathcal{M}}}
\def\gN{{\mathcal{N}}}
\def\gO{{\mathcal{O}}}
\def\gP{{\mathcal{P}}}

\def\gR{{\mathcal{R}}}
\def\gS{{\mathcal{S}}}
\def\gT{{\mathcal{T}}}

\def\gW{{\mathcal{W}}}

\def\gZ{{\mathcal{Z}}}

\def\sN{{\mathbb{N}}}

\def\sP{{\mathbb{P}}}

\def\sR{{\mathbb{R}}}

\newcommand{\VC}{\textsf{VCdim}}
\newcommand{\holder}{H\"{o}lder }

\title{Neural Network Approximation for Pessimistic Offline Reinforcement Learning}
\author{
    Di Wu \textsuperscript{\rm 1}, Yuling Jiao \textsuperscript{\rm 1, \rm 2}, Li Shen \textsuperscript{\rm 3}, Haizhao Yang \textsuperscript{\rm 4}\footnote{Correspondence to Xiliang Lu or Haizhao Yang.}, Xiliang Lu \textsuperscript{\rm 1, \rm 2 \textasteriskcentered}
}
\affiliations{
    \textsuperscript{\rm 1} School of Mathematics and Statistics, Wuhan University, China, \\
    \textsuperscript{\rm 2} Hubei Key Laboratory of Computational Science, Wuhan University, China,
    \textsuperscript{\rm 3} JD Explore Academy, China, \\
    \textsuperscript{\rm 4} Department of Mathematics and Department of Computer Science, University of Maryland, College Park, USA,

    \textit{\{dwu.math,yulingjiaomath,xllv.math\}@whu.edu.cn,
    mathshenli@gmail.com,
    hzyang@umd.edu}
}

\usepackage{bibentry}

\begin{document}

\maketitle

\begin{abstract}
Deep reinforcement learning (RL) has shown remarkable success in specific offline decision-making scenarios, yet its theoretical guarantees are still under development. Existing works on offline RL theory primarily emphasize a few trivial settings, such as linear MDP or general function approximation with strong assumptions and independent data, which lack guidance for practical use. The coupling of deep learning and Bellman residuals makes this problem challenging, in addition to the difficulty of data dependence. In this paper, we establish a non-asymptotic estimation error of pessimistic offline RL using general neural network approximation with $\mathcal{C}$-mixing data regarding the structure of networks, the dimension of datasets, and the concentrability of data coverage, under mild assumptions. Our result shows that the estimation error consists of two parts: the first converges to zero at a desired rate on the sample size with partially controllable concentrability, and the second becomes negligible if the residual constraint is tight. This result demonstrates the explicit efficiency of deep adversarial offline RL frameworks. We utilize the empirical process tool for $\mathcal{C}$-mixing sequences and the neural network approximation theory for the H\"{o}lder class to achieve this. We also develop methods to bound the Bellman estimation error caused by function approximation with empirical Bellman constraint perturbations. Additionally, we present a result that lessens the curse of dimensionality using data with low intrinsic dimensionality and function classes with low complexity. Our estimation provides valuable insights into the development of deep offline RL and guidance for algorithm model design.
\end{abstract}

\section{Introduction}

Online RL has demonstrated significant empirical success in specific decision-making problems \cite{mnih2015human,silver2016mastering}. However, numerous real-world environments only allow limited interaction, presenting challenges in cost (e.g., robotics) or safety (e.g., autonomous driving). Therefore, offline RL was introduced as a paradigm that enables learning decision-making problems from pre-collected data without additional interaction \cite{lange2012batch,levine2020offline}. The primary objective of offline RL is to leverage the available data to learn near-optimal policies even when the data is insufficiently collected.

Distribution shift is a significant challenge that offline RL faces due to inconsistency between the data used for training and the data induced by the learned policy. One way to overcome the distribution shift challenge is to enforce a policy constraint \cite{fujimoto2019off,kumar2019stabilizing}. However, this approach can be overly conservative and is highly dependent on the accuracy of estimating the behavior policy. To avoid overestimation and focus on reliable data distribution, several studies \cite{kumar2020conservative,xie2021bellman,cheng2022adversarially} have modified the Q-value function, enabling the selection of actions in a pessimistic manner. These two approaches are known as regularized policy-based and pessimistic value-based methods, respectively. The efficacy of these methods has been verified in complex offline RL environments \cite{fu2020d4rl}.

\begin{table}
    \centering
    \resizebox{0.46\textwidth}{!}{
    \begin{tabular}{c|c|c}
    \toprule
    Existing Works & Assumption & Coverage \\
    \midrule
\citet{szepesvari2005finite,munos2007performance} & \multirow{6}{*}{$\backslash$} & \multirow{8}{*}{Full} \\
\citet{antos2007fitted,antos2008learning} & {} & {} \\
\citet{farahmand2010error} & {} & {} \\
\citet{scherrer2014approximate,liu2019neural,chen2019information} & {} & {} \\
\citet{jiang2019value,wang2019neural,feng2019kernel} & {} & {} \\
\citet{liao2022batch,zhang2020variational} & {} & {} \\
\citet{uehara2020minimax,xie2021batch} & {} & {} \\ \cmidrule{1-2}
\citet{nguyen-tang2022on} & {Neural network} & {} \\ \cmidrule{1-3}
\citet{rashidinejad2021bridging,yin2021near} & \multirow{2}{*}{Tabular MDP} & \multirow{12}{*}{Partial} \\ 
\citet{shi2022pessimistic,li2022settling} & {} & {} \\ \cmidrule{1-2}
\citet{jin2021pessimism,chang2021mitigating} & \multirow{2}{*}{Linear MDP} & {} \\
\citet{zhang2022corruption,nguyen2022instance,bai2022pessimistic} & {} & {} \\ \cmidrule{1-2}
    \citet{jiang2020minimax} & Compact & {} \\ \cmidrule{1-2}
    \citet{zhan2022offline} & Strongly convex & {} \\ \cmidrule{1-2}
    \citet{uehara2020minimax,rashidinejad2022optimal} & \multirow{2}{*}{Finite} & {} \\
    \citet{zanette2022bellman,xie2021bellman,cheng2022adversarially} & {} & {} \\ \cmidrule{1-2}
    \citet{ji2023sample} & \multirow{2}{*}{Neural network} & {} \\
    Our work & {} & {} \\
    \bottomrule
    \end{tabular}}
    \caption{A comparison of existing works concerning assumptions related to data coverage and approximation.}
    \label{tab:my_label}
\end{table}
The empirical success of recent studies far surpasses the progress made in the theoretical framework for offline RL. Early works \cite{szepesvari2005finite,munos2007performance,antos2007fitted,antos2008learning} assume data to be fully covered, which is unrealistic. More recent studies have relaxed this assumption to partial coverage, focusing mainly on tabular and linear function approximations \cite{jin2021pessimism,chang2021mitigating,zhang2022corruption,nguyen2022instance,bai2022pessimistic,rashidinejad2021bridging,yin2021near,shi2022pessimistic,li2022settling}. General function approximations have been investigated in \citet{jiang2020minimax,uehara2021pessimistic,zhan2022offline,rashidinejad2022optimal,zanette2022bellman,xie2021bellman,cheng2022adversarially}, but they still rely on additional assumptions such as finiteness and convexity. See Table \ref{tab:my_label} for a comparison to prior works. Practical applications often adopt deep neural network parameterization, which is highly non-convex. Additionally, the approximation of policy function is often ignored even in the actor-critic framework \cite{xie2021bellman,cheng2022adversarially}. Another challenge for offline RL theory is that data is sequentially dependent. In contrast, current works assume offline data to be independent and identically distributed (i.i.d.), leading to a statistical deviation not aligning with reality. In summary, existing offline RL theory faces three main challenges: overly strong assumptions regarding the value function, the inadequate inclusion of policy function approximation, and the neglect of data dependence.

This paper investigates the performance of a deep adversarial offline RL framework that uses deep neural networks to parameterize both the value and policy functions while assuming the data to be dependent. 
Our result indicates the estimation error consists of two parts: the first converges to zero at a desired rate on the sample size, and the second decreases if the residual constraint is tight. We demonstrate that the estimation rate depends explicitly on the network structure, the tightness of the Bellman constraint, and partially controllable data coverage. However, the curse of dimensionality presents a challenge to the result, mainly when data dimensions are enormous. To alleviate this, we propose using low-dimensional data structures or low-complexity target functions consistent with real-world scenarios.

This paper represents the first attempt to bridge the gap between theory and practice by providing a theoretical analysis in the context of deep pessimistic offline RL. Main contributions are summarized as follows:
\begin{itemize}[leftmargin=*]
\item We reassess offline RL methods and present an adversarial framework that utilizes deep neural networks to parameterize both policy and value functions, where the offline data is sequentially dependent with only partial coverage.
\item We establish a non-asymptotic rate for the estimation error of deep adversarial offline RL regarding the width and depth of networks, dataset dimensions, and the concentrability of distribution shift. Our results are derived under mild assumptions and explicitly illustrate how the choices of neural network structure and algorithm setting influence the efficiency of deep offline RL.
\item We mitigate the curse of dimensionality by utilizing low-dimensional data structures or low-complexity target functions, providing an idealized guarantee for real-world data.
\end{itemize}
\subsubsection{Technical Contribution.} Our technical challenge stems from multiple practical considerations, which we clarify from four distinct perspectives. (1) Confining adversarial terms presents a non-trivial task, particularly when optimized within different constrained sets. We introduce an operator perturbation analysis, which holds under $L_p$ norm with distribution shifts. (2) Part of generalization errors is rooted in constraints rather than explicit functions. We leverage a generalized version of performance difference to disentangle them from constraints. (3) Due to dependent data, parameterized value and policy functions, the induced space is complex. We tackle this by using uniform covering numbers (focusing on samples, not the entire class), ghost point analysis, and an extended Bernstein inequality. (4) We alleviate the curse of dimensionality with two considerations. Firstly, Minkowski dimension can measure the dimensionality of highly irregular sets like fractals, where we employ Whitney’s extension theorem to establish a novel bound. Secondly, our work is the first theoretical attempt to consider low complexity in RL, where we develop perturbation and recursion analysis with additional connecting layers.

\section{Related Works}
\label{sec:relate}

\noindent
\subsubsection{Offline RL.} 
Recent advances in offline RL can be categorized into two groups. The first group, called policy-based regularization, involves directly constraining the learned policy to be similar to the behavior policy \cite{fujimoto2019off,laroche2019safe,kumar2019stabilizing,siegelkeep}. However, these methods suffer from two issues: (a) they may be overly conservative, similar to behavior cloning, and (b)  estimating the behavior policy can be challenging. Instead of directly constraining the policy, the second group modifies the learning objective to avoid overestimating the value function \cite{kumar2020conservative,liu2020provably,jin2021pessimism,kostrikov2021offline,uehara2021pessimistic,xie2021bellman,cheng2022adversarially,rigterrambo,bhardwaj2023adversarial}. For instance, \citet{kumar2020conservative} and \citet{kostrikov2021offline} utilize a conservative approach to optimize the lower bound of the value function, ensuring safe improvement. On the other hand, \citet{xie2021bellman,cheng2022adversarially} introduce a bilevel scheme to emphasize Bellman-consistent pessimism, while \citet{rigterrambo,bhardwaj2023adversarial} employ an adversarial MDP model to minimize policy performance. 

Although recent methods in offline RL have demonstrated impressive empirical results, theoretical foundations are not well understood. Early studies of offline RL theory are analyzed with strong assumptions, such as full data coverage \cite{szepesvari2005finite,munos2007performance,antos2007fitted,antos2008learning,farahmand2010error,scherrer2014approximate,liu2019neural,chen2019information,jiang2019value,wang2019neural,feng2019kernel,liao2022batch,zhang2020variational,uehara2020minimax,xie2021batch}. Recent analyses relax this assumption to partial coverage. \citet{rashidinejad2021bridging,yin2021near,shi2022pessimistic,li2022settling} have studied the tabular MDP, while \citet{jin2021pessimism,chang2021mitigating,zhang2022corruption,nguyen2022instance,bai2022pessimistic} explored the linear MDP. General function approximation has been studied in \cite{jiang2020minimax,uehara2021pessimistic,zhan2022offline,rashidinejad2022optimal,zanette2022bellman,xie2021bellman,cheng2022adversarially}. Specifically, \citet{jiang2020minimax} assume the value function class to be compact and explore the convex hull. \citet{zhan2022offline} assume the value function to be strongly convex, while \citet{uehara2020minimax,rashidinejad2022optimal,zanette2022bellman,xie2021bellman} and \citet{cheng2022adversarially} assume the function class to be finite. However, both value and policy functions are nonconvex and infinite in reality, which is the main concern of this study. Furthermore, most of theoretical works are analyzed in the i.i.d. setting, which does not reflect data dependence.

\noindent
\subsubsection{Approximation and Generalization in Deep Learning.}
Extensive research has examined the estimation error of deep learning (DL) and how it guides the training process. This error typically comprises two components: approximation, which evaluates the expressive power of deep neural networks for specific general functions, and generalization, which measures the deviation between finite data samples and the expectation. The approximation theory of deep learning has been studied for continuous functions \cite{shen2019deep,yarotsky2021elementary,shen2021deep} and smooth functions \cite{yarotsky2017error,yarotsky2018optimal,suzuki2018adaptivity,lu2021deep,suzuki2021deep,jiao2023approximation}. Meanwhile, the generalization theory of deep learning has been extensively explored in the context of i.i.d. data \cite{anthony1999neural,schmidt2020nonparametric,nakada2020adaptive,bauer2019deep,farrell2021deep,jiao2023deep}. For dependent data, statistical techniques have been the focus of research in \citet{yu1994rates,antos2008learning,hang2017bernstein,steinwart2009learning,mohri2008rademacher,mohri2010stability,ralaivola2015entropy,roy2021empirical}. This paper presents the first analysis in the context of pessimistic offline RL problems with dependent data and deep neural network approximation.

\begin{table}
    \centering
    \vspace{-0.15cm}
    \resizebox{0.46\textwidth}{!}{
    \begin{tabular}{c|c|c}
    \toprule
    Work & Method & Curse of dimensionality? \\
    \midrule    \citet{nguyen-tang2022on} & {OPE/OPL} & {Exist} \\ \cmidrule{1-3}
    \citet{ji2023sample} & {OPE} & {Low-D Riemannian manifold} \\ \cmidrule{1-3}
    \multirow{2}{*}{This work} & \multirow{2}{*}{Adversarial} & {General Minkowski dimension} \\
    {} & {} & {Low Complexity} \\
    \bottomrule
    \end{tabular}}
    \caption{A comparison of works using network approximation. OPE/OPL is for off-policy evaluation/learning.}
    \label{tab:my_label2}
    \vspace{-0.4cm}
\end{table}

\noindent 
\subsubsection{Additional Related Works.} Recently, two studies have offered theoretical insights into deep offline reinforcement learning. \citet{nguyen-tang2022on} analyze the sample complexity associated with offline policy evaluation and optimization using a deep ReLU neural network approximation. However, their findings are afflicted by the curse of dimensionality and necessitate complete data coverage. \citet{ji2023sample} investigate the estimation error of the fitted Q-evaluation method employing convolutional neural networks. They introduce a novel concentrability metric and mitigate the curse of dimensionality by conceptualizing the data space as a low-dimensional Riemannian manifold.

While these two studies concentrate on the variant of the fitted Q-iteration under i.i.d. setting, our work focuses on the pessimistic approach with sequentially dependent data, necessitating a more intricate analysis due to the inclusion of uncertainty quantifiers and adversarial strategies.  Furthermore, we address the curse of dimensionality by utilizing a general measure of dimensionality and target functions possessing low complexity. See Table \ref{tab:my_label2} for a clear comparison.

In addition, several recent works consider RL allowing time-dependence \cite{zou2019finite,kallus2022efficiently,shi2022statistical}, and developing pessimistic-type algorithms \cite{lyu2022mildly,zhou2023optimizing}. However, these works focus on less practical assumptions or empirical performance, which are quite different from our concerns.

\section{Preliminaries}
\label{sec:prelims}
 

\subsection{Reinforcement Learning}

\noindent
\subsubsection{Markov Decision Processes (MDPs).}
This work considers a discounted MDP, defined with a tuple $(\gS,\gA,P,R,\gamma)$, where $\gS$ is the set of states, $\gA$ is the set of actions, and $\gamma\in(0,1)$ is the discount factor. $P:\gS\times\gA\to\Delta(\gS)$ is the Markov transition kernel, and $R:\gS\times\gA\to\Delta(\sR)$ is the immediate reward. Given a specific pair $(s,a)\in\gS\times\gA$, $P(\cdot|s,a)$ refers to the probability distribution of the next state, and $R(\cdot|s,a)$ refers to the probability distribution of the immediate reward. For regularity, the reward is assumed to be bounded by $R_{\max}$, and the MDP starts at the initial state $s_0$. A policy $\pi:\gS\to\Delta(\gA)$ is used to decide which action to take, and accordingly, a sequence is obtained as
\begin{small}
\begin{equation*} a_t\sim\pi(\cdot|s_t),r_t\sim R(\cdot|s_t,a_t),s_{t+1}\sim P(\cdot|s_t,a_t).
\end{equation*}  
\end{small}

\noindent RL aims to find the optimal policy $\pi^*$ maximizing the value function $V^{\pi}(s)$, the expected cumulative discounted reward starting from $s$, i.e., $V^{\pi}(s)=\Expect\left[\sum_{t=0}^{\infty}\gamma^t r_t|s_0=s\right]$. Similarly, we define the action-value function $Q^{\pi}(s,a)=\Expect\left[\sum_{t=0}^{\infty}\gamma^t r_t|s_0=s,a_0=a\right]$ starting from $s$, taking action $a$ and then following policy $\pi$.
The boundedness of rewards guarantees $V^{\pi}(s)$ and $Q^{\pi}(s,a)$ are both in $[0, R_{\max}/(1-\gamma)]$. We define the Bellman operator as follows:
\begin{small}
\begin{equation*}    \gT^{\pi}Q(s,a)=\Expect[R(s,a)]+\gamma P^{\pi}Q(s,a),
\end{equation*}\end{small}
\noindent where $P^{\pi}Q(s,a):=\int P(ds'|s,a)\pi(da'|s)Q(s',a')$. It has been proven that the Bellman operator is contractive concerning the sup-norm \cite{sutton2018reinforcement}, i.e., 
\begin{small}\begin{equation*}
 \|\gT^{\pi} Q_1 - \gT^{\pi} Q_2\|_{\infty}\leq \gamma \|Q_1-Q_2\|_{\infty}
\end{equation*}\end{small}
\noindent for any two action-value functions $Q_1$ and $Q_2$. This property guarantees the existence of a fixed point $Q^{\pi}$ with respect to $\gT^{\pi}$, which induces the value iteration algorithm.

\noindent
\subsubsection{Offline RL.}
Offline RL aims to learn an optimal policy using a given dataset without interacting with the environment. The fixed dataset $\gD$ consists of tuples ${s,a,r,s'}$ with $s$ and $a$ sampled from the state-action distribution of a behavior policy $\mu$, $r$ and $s'$ induced by the environment. For any policy $\pi$, we define the marginal state-action occupancy measure as $\rho^{\pi}$. We also denote $\mu = \rho^{\mu}$, with a slight abuse of notation.

\begin{defn} [Concentrability Coefficient] \label{defn:concentrability}
	Let $\mu$ be the behavior policy and $\pi$ be a comparator policy; define the density ratio based concentrability coefficient as follows:
	\begin{small}\begin{equation*}
	\gC(\pi;\mu):=\sup_{(s,a)}\frac{\rho^{\pi}(s,a)}{\mu(s,a)}.
	\end{equation*}\end{small}
\end{defn}
\noindent This definition of concentrability is widely used in literature \cite{szepesvari2005finite,munos2007performance,chen2019information,xie2020q}, and \citet{chen2019information} also offer rich practical insights, indicating the presence of low concentrability. The definition of the concentrability coefficient varies in a few kinds of literature, such as full coverage in \cite{szepesvari2005finite}, Bellman residual-based perspective in \cite{xie2021bellman,cheng2022adversarially}, and $\chi^2$-divergence in \cite{ji2023sample}. Our result could potentially be extended to a tighter metric, e.g., Bellman residual-based, involving the separation of on/off support parts. Nonetheless, this distinction does not fall within the primary scope of our study. For a comprehensive review of concentrability, refer to \citet{uehara2021pessimistic} and the references therein.

This work assumes only partial coverage within several particular policies, which we will explain in Section \ref{sec:results}. We extensively utilize $\gC(\Pi;\mu):=\sup_{\pi\in\Pi}\gC(\pi;\mu)$ to represent the concentrability of a set of policies with respect to $\mu$.

\subsection{Feed-Forward Deep Neural Networks}
 
In this work, our primary focus is on the multi-layer
feed-forward neural network (FNN) activated by the rectified linear unit (ReLU) function $\sigma(x)=\max\{0,x\}$ with $x\in\sR^{d}$:
\begin{align*}
    f_0(x)&=x, \\
    f_{\ell}(x)&=\sigma\left(W_{\ell} f_{\ell-1}(x)+b_{\ell}\right), \quad \ell=1,\ldots,L-1, \\
    f(x)&=f_{L}(x)=W_{L}f_{L-1}(x)+b_L.
\end{align*}
Here $W_{\ell}\in\sR^{n_{\ell}\times n_{\ell-1}}, n_0=d$ and $b_{\ell}\in\sR^{n_{\ell}}$ are the weight parameters of the $\ell$ layer. The activation function $\sigma$ is applied entry-wise. A network with width $\gW$ and depth $\gL$ means $\gW=\max\{n_{\ell},\ell=0,\ldots,L\}, \gL=L-1$. That is, the maximum width of the hidden layers does not exceed $\gW$, and the number of the hidden layers does not exceed $\gL$. The weight parameters consist of $W_{\ell},b_{\ell},\ell=0,\ldots,L$, and we denote the total number of parameters as $\gP$. 
For simplicity, we may use $\gN\gN$ to denote  ReLU FNNs in this work.

Notations in this paper are summarized in Appendix A.
 

\section{Main Results} \label{sec:results}
 

\subsection{Adversarial Offline RL}
 
Adversarial offline RL has been extensively studied in literature \cite{kumar2020conservative,xie2021bellman,cheng2022adversarially,rigterrambo,bhardwaj2023adversarial}. We formulate the framework with relative pessimism \cite{cheng2022adversarially} as a maximization-minimization problem:
\begin{small}\begin{equation} \label{maximin}
	\widehat{\pi}^{*} \in \mathop{\arg\max}_{\pi\in\Pi} \mathop{\min}_{f\in\gF_{\mu}^{\pi,\epsilon}}\gL_{\mu}(\pi,f),
\end{equation}\end{small}
with $\gL_{\mu}(\pi,f):=\Expect_{\mu}[f(s,\pi)-f(s,a)]$, $\gF_{\mu}^{\pi,\epsilon}:=\{f\in\gF\mid\gE_{\mu}(\pi,f)\le \epsilon\}$ with 
$\gE_{\mu}(\pi,f):=\|f-\gT^{\pi}f\|_{2,\mu}^2$. Here, 
 $\gF$ is the set of functions $f:\gS\times\gA\to[0,V_{\max}]$ and $\Pi$ is the policy function class. Practically, this constrained pessimism framework is implemented by adversarial regularized algorithms as introduced in \citet{bhardwaj2023adversarial}, to approximately address a specific sub-question. These algorithms have exhibited competent performance across diverse offline scenarios owing to the robust improvement over uncertainty. In this study, we focus on the essential max-min problem, leaving the algorithm analysis as future directions.
 
The population-level problem (\ref{maximin}) is intractable because the oracle distribution is not accessible. To solve the optimization problem (\ref{maximin}), we propose an empirical scheme that can be used for computation with neural network approximation. First, we define an estimated Bellman error:
 \begin{small}\begin{align*}
 \gE_{D}(\pi,f):=&\Expect_{\gD}\left[(f(s,a)-r-\gamma f(s',\pi))^{2}\right] \\
 &-\min_{f'\in\gF}\Expect_{\gD}\left[(f'(s,a)-r-\gamma f(s',\pi))\right],
 \end{align*}\end{small}
which is shown to be an unbiased estimation of $\gE_{\mu}(\pi,f)$ in \citet{antos2008learning}.
Consider $\Pi_{\theta}$ as the set of parameterized policies $\{\pi_{\theta} \mid \theta \in \Theta \subseteq \mathbb{R}^{d}\}$. Actions are selected based on the likelihood derived from the probability density function (PDF) of the policy distribution. This PDF is approximated through a ReLU FNN in $\gN\gN_1$. Importantly, the approximation need not strictly adhere to being a specific density. This flexibility arises from the possibility of drawing samples directly from the density, for instance, using kernel density estimation \cite{rosenblatt1956remarks}. The following equation gives the computation scheme we consider:
\begin{small}\begin{equation} \label{maximin_emprical}
     \widehat{\pi} = \mathop{\arg\max}_{\pi\in\Pi_{\theta}} \mathop{\min}_{f\in \gN\gN_{2}\cap \gF_{\gD}^{\pi,\epsilon}}\gL_{\gD}(\pi,f)
\end{equation}\end{small}
where $\gL_{\gD}(\pi,f):=\Expect_{\gD}[f(s,\pi)-f(s,a)], \quad \gF_{\gD}^{\pi,\epsilon}:=\{f\in\gF\mid\gE_{\gD}(\pi,f)\le \epsilon\}$, and $\gN\gN_2$ refer to a ReLU FNN used to approximate the value function. To simplify the theoretical analysis, we also use $\gR_{\mu}(\pi,f),\gR_{\gD}(\pi,f)$ to denote $-\gL_{\mu}(\pi,f)$ and $-\gL_{\gD}(\pi,f)$, respectively. Thus, problems (\ref{maximin}) and (\ref{maximin_emprical}) can be reformulated as minimax problems:
\begin{small}\begin{align} \label{minimax}
&\widehat{\pi}^{*} \in \mathop{\arg\min}_{\pi\in\Pi_{\theta}} \mathop{\max}_{f\in\gF_{\mu}^{\pi,\epsilon}}\gR_{\mu}(\pi,f), \\
    &\widehat{\pi} \in \mathop{\arg\min}_{\pi\in\Pi_{\theta}}\mathop{\max}_{f\in\gN\gN_2\cap\gF_{\gD}^{\pi,\epsilon}}\gR_{\gD}(\pi,f).
\end{align}\end{small}

\noindent There exists a gap between $\widehat{\pi}$ and the exact solution $\widehat{\pi}^{*}$, due to finite sampling and imperfect approximation. We aim to explicitly measure this gap concerning the network structure and data sampling, which guides the training process. We want to emphasize that the analysis of this problem is quite challenging due to the coupling of network approximation and the empirical constraint of the Bellman error.
 
\subsection{Technical Assumptions}

In DL and RL theory, several mild assumptions are commonly utilized. We define the \holder function class and $\gC$-mixing process \cite{maume2006exponential} as follows.
\begin{defn}[\holder Smooth Function Class]
For $\zeta=s+r$ with $s\in\sN^{+}$ and $0<r\leq 1$,  the \holder smooth function class $\gH^{\zeta}$ is defined as
    \begin{small}\begin{align*}
         \gH^{\zeta}=\Big\{f:[0,1]^{d}&\to\sR \Big|\max_{\|\alpha\|_{1}\leq s}\|\partial^{\alpha}f\|_{\infty}\leq B, \\ &\max_{\|\alpha\|_{1}=s}\sup_{x\neq y}\frac{|\partial^{\alpha}f(x)-\partial^{\alpha}f(y)|}{\|x-y\|_{\infty}^{r}}\leq B\Big\}.
    \end{align*}\end{small}
\end{defn}
\begin{defn}[$\gC$-mixing process]
Let $(\Omega,\gA,\mu)$ be a probability space, $(Z,\gB)$ be a measurable space, and $\gZ:=(Z_i)_{i\geq 0}$ be a $Z$-valued stationary process on $\Omega$. For any $n\geq 0$, we define the $\gC$-mixing coefficients as
\begin{small}\begin{align*}
\psi_{\gC}(\gZ,n):=\sup\{&\mathrm{cor}(Y,h\circ Z_{k+n}): \\
& k\geq 0, Y\in B_{L_{1}(\gA_{0}^{k},\mu)},h\in B_{\gC(Z)}\}, 
\end{align*}\end{small}
where $\mathrm{cor}(\cdot,\cdot)$ denotes the correlation of two random variables, i.e., $\mathrm{cor}(X,Y)=\Expect[XY]-\Expect[X]\Expect[Y]$ if $X,Y,XY\in L_{1}(\Omega,\gA,\mu)$. $\gA_{0}^{k}$ is the $\sigma$-algebra generated by $(Z_0,\ldots,Z_k)$ and $\gC(Z)$ is the bounded function space $\{f:Z\to\sR\mid\|f\|_{\infty}+\|f\|<\infty\}$ where $\|\cdot\|$ is a semi-norm.
\end{defn}

\noindent Additionally, if we have $\psi_{\gC}(\gZ,n)\leq d_n$ for all $n>0$, where $(d_{n})_{n\geq 0}$ is a strictly positive sequence converging to 0, then $\gZ$ is said to be $\gC$-mixing with rate $(d_{n})_{n\geq 0}$. If $(d_{n})_{n\geq 0}$ is of the form $d_n=c \exp(-bn^{\eta})$ for $b>0$, $c\geq 0$, and $\eta>0$, then $\gZ$ is called geometrically $\gC$-mixing. 

\begin{asmp}[Smoothness] \label{asmp:holder}
    Without loss of generality, we assume $z_i:=\{s_i,a_i\}\in [0,1]^d$. Density functions of policies in $\Pi$ and value functions in $\gF$ are \holder smooth.
\end{asmp}
\begin{asmp}[Completeness] \label{asmp:complete}
    For any $\pi\in\Pi_{\theta}$ and $f\in\gN\gN_2$, we have $\gT^{\pi}f\in \gF$.
\end{asmp}
\noindent Assumption \ref{asmp:holder} is a generalization of Lipschitz continuity. It is commonly used in theory and efficient in capturing real-world features \cite{fan2020theoretical}. Assumption \ref{asmp:complete} holds when rewards and values belong to smooth function classes \cite{fan2020theoretical}. Moreover, \citet{chen2019information,wang2021statistical} verify that completeness is indispensable even in simple scenarios. While Assumption 4.1 emphasizes the smoothness, it may not sufficiently guarantee the completeness stated in Assumption 4.2. 
\begin{asmp}[Mixing] \label{asmp:mixing}
    We assume the batch data $\{s_t,a_t,r_t\}_{t\geq 0}$ satisfies the definition of strictly stationary geometrically $\gC$-mixing process with parameters $b,c,\eta>0$.
\end{asmp}

\noindent Assumption \ref{asmp:mixing} describes the mixing rate for the batch data sequence and indicates that the future weakly depends on the past. This property of weak dependence is general, encompassing $\phi$-mixing \cite{ibragimov1962some} as a particular case and overlapping with $\alpha$-mixing \cite{rosenblatt1956central}. The quantitative distinctions between $\alpha$-mixing and $\gC$-mixing are examined in \citet{hang2016learning}. Experiments \cite{solowjow2020kernel} illustrate that mixing captures the autocorrelation speed of dynamical systems, including Markov chains as a specific case, which characterizes the essential nature of data dependence. We notice that several studies focus on episodic MDPs \cite{jin2021pessimism}, but within a linear setting rather than a universal characterization.
 
\subsection{Loss Consistency}
 
Recalling the problem (\ref{minimax}), we denote the risk of $\pi$ as
\begin{small}\begin{equation} \label{defn:Rtilde}
\widetilde{\gR}_{\mu}(\pi,\epsilon)=\max_{f\in \gF_{\mu}^{\pi,\epsilon}} \gR_{\mu}(\pi,f).
\end{equation}\end{small}
Our first main theorem presents an upper bound for the excess risk, $\widetilde{\gR}_{\mu}(\widehat{\pi},\epsilon)-\widetilde{\gR}_{\mu}(\widehat{\pi}^{*},\epsilon)$, which quantitatively measures the difference between $\widehat{\pi}$ and $\widehat{\pi}^{*}$, and demonstrates the efficacy of the adversarial offline RL framework (\ref{maximin}).

\begin{thm} \label{thm:excess_risk}
    Under Assumptions \ref{asmp:holder},\ref{asmp:complete} and \ref{asmp:mixing}, let $\widehat{\pi}$ and $\widehat{\pi}^{*}$ be defined in (\ref{minimax}). Then, for $\gN\gN_1$, $\gN\gN_2$ with width $\gW=\gO(d^{s+1}|\gD|^{\frac{d}{2d+4\zeta^*}})$ and depth $\gL=\gO(\log(|\gD|))$, the following non-asymptotic error bound holds 
    \begin{small}\begin{align*}
    &\quad\Expect[\widetilde{\gR}_{\mu}(\widehat{\pi},\epsilon)-\widetilde{\gR}_{\mu}(\widehat{\pi}^{*},\epsilon)] \\ &\leq C_{1}R_{\max}d^{s+(\zeta\lor 1)/2}|\gD|^{\frac{-\zeta^*}{d+2\zeta^*}}\log(|\gD|)^{2+\frac{1}{\eta}}+C_{2}\sqrt{\epsilon}, \end{align*}\end{small}
    where $\zeta^*=\zeta(1\land\zeta)$, $C_{1}$ is a constant depending on $s,B,\gC(\widehat{\pi};\mu),\gC(\widehat{\pi}_{\delta}^{*};\mu)$ and $C_{2}$ is a constant depending on $\gC(\widehat{\pi};\mu),\gC(\widehat{\pi}_{\delta}^{*};\mu)$.
\end{thm}

\noindent A small constant $C_1$ is achieved when both $\widehat{\pi}$ and $\widehat{\pi}_{\delta}^{*}$ demonstrate effective controllability concerning the behavior policy $\mu$. Specifically, $\widehat{\pi}_{\delta}^{*}$ represents a $\delta$-neighborhood of $\widehat{\pi}^{*}$ in terms of their densities, with $\delta$ being the approximation error. In other words, our assumption pertains to the partial controllability of data coverage related to $\widehat{\pi}$ and a small area around $\widehat{\pi}^{*}$. Nevertheless, when concentrability is poor, $C_1$ might be very large, aligning with the unfavorable empirical results under challenging distribution shifts (Levine et al., 2020). A larger value of $\zeta$ indicates a faster order, implying that estimating a smoother target is more manageable. The non-asymptotic bound $\gO(|\gD|^{\frac{-\zeta^*}{d+2\zeta^*}})$ is afflicted by the curse of dimensionality, a concern we tackle in the next section.

Under mild conditions, the optimal rate in nonparametric regression is $C_d |\gD|^{-2\zeta/(2\zeta+d)}$ \cite{stone1982optimal}, which aligns with ours. Moreover, our prefactor is polynomial in $d$ instead of exponential \cite{shen2019deep}. These results are tight and new in RL. The optimality is also extensively discussed in \citet{suzuki2018adaptivity,suzuki2021deep}.

The hyperparameter $\epsilon$ in the second term corresponds to the Bellman constraint, as introduced in (\ref{minimax}), which is restricted by the expressive capacity of the value function class, but may still be small. Its prefactor is also related to the concentrability of $\widehat{\pi}$ and $\widehat{\pi}_{\delta}^{*}$. The constraint $\epsilon$ significantly influences the training process, as it balances the accuracy and uncertainty aspects of the acquired value function. Note that $\epsilon$ should be at least larger than the gap between the value function space and the corresponding Bellman mapping space, but still dominated by the first term in the bound.

This explicit bound has no unknown parameters involved, including the width and depth of the network, providing informative guidance for training adversarial offline RL. By selecting suitable width and depth for neural networks, the estimation error exhibits an exponential decrease as the number of data samples increases. This result matches the empirical observation of network approximation \cite{montufar2014number} and generalization \cite{novak2018sensitivity}.

\subsection{Circumvent the Curse of Dimensionality} \label{sec:curse_of_dim}
 
Theorem \ref{thm:excess_risk} indicates a curse of dimensionality when the data dimension is large. According to the ``no free lunch'' theorem \cite{wolpert1996lack}, any method regardless of data or model conditions is susceptible to this challenge. To mitigate this, we aim to alleviate the curse of dimensionality by utilizing a priori information under two scenarios:
\begin{itemize}[leftmargin=*]
    \item Data structure with low Minkowski dimension.
    \item Target function combined of low-complexity elements.
\end{itemize}

\subsubsection{Low-Dimensional Data Structure.}
We start with the definitions of covering numbers \cite{vershynin2018high} and Minkowski dimension \cite{bishop2017fractals}.
\begin{defn}[Covering Number]
	Let $\dist$ be a metric, $\epsilon>0$ and $K\subset \sR^{n}$. A subset $\gN\subset K$ is an $\epsilon$-net of K if 
	\begin{small}\[ \forall x\in K, \exists x_0\in\gN:\dist(x,x_0)\leq\epsilon. \]\end{small}
The smallest cardinality of an $\epsilon$-net of $K$ is called the covering number of $K$, denoted by $\gN(K,\dist,\epsilon)$.
\end{defn}

\begin{defn}[Minkowski Dimension] \label{defn:minkowski}
	Let $\dist$ be a metric, $\epsilon>0$ and $K$ be a subset of $\sR^{n}$, i.e., $K\subset \sR^{n}$. We define the upper and lower Minkowski dimensions as
	\begin{small}\begin{align*}
	    \overline{\mathrm{dim}}_{\gM}(K)&=\mathop{\lim\sup}_{\epsilon\to 0}\frac{\log\gN(K,\dist,\epsilon)}{-\log(\epsilon)}, \\ \underline{\mathrm{dim}}_{\gM}(K)&=\mathop{\lim\inf}_{\epsilon\to 0}\frac{\log\gN(K,\dist,\epsilon)}{-\log(\epsilon)}.
	\end{align*}\end{small}
Furthermore, if $\overline{\mathrm{dim}}_{\gM}(K)=\underline{\mathrm{dim}}_{\gM}(K)$, this value is called the Minkowski dimension and denoted by $\mathrm{dim}_{\gM}(K)$.
\end{defn}
\noindent Obviously, $\gN(K,\dist,\epsilon)=\epsilon^{-\mathrm{dim}_{\gM}(K)+o(1)}$. This indicates that the Minkowski dimension measures the decay rate in covering numbers as $\epsilon$ tends towards 0.

\begin{rem}
    For any manifold, its Minkowski dimension is equivalent to its dimension. Although the high ambient dimensions of real-world data are quite large, such as those in MNIST \cite{lecun1998gradient}, CIFAR \cite{krizhevsky2009learning}, ImageNet \cite{deng2009imagenet}, the intrinsic dimensions have been estimated to be relatively low \cite{recanatesi2019dimensionality,popeintrinsic}. Hence, it is reasonable to assume that the data has a low-dimensional structure, indicating that it is supported by a space with a small Minkowski dimension.
\end{rem}
\begin{thm} \label{thm:low_dim}
    Suppose that the support of $\gS\times\gA$ is $K\subset[0,1]^d$, and its Minkowski dimension satisfies $\mathrm{dim}_{\gM}(K)\ll d$. Assuming Assumptions \ref{asmp:holder},\ref{asmp:complete} and \ref{asmp:mixing} hold, we define $\widehat{\pi}$ and $\widehat{\pi}^{*}$ as in (\ref{minimax}). Then, for $\gN\gN_1$ and $\gN\gN_2$ with width $\gW=\gO(d_K^{s+1}|\gD|^{\frac{d_K}{2d_K+4\zeta^*}})$ and depth $\gL=\gO(\log(|\gD|))$, we can establish a non-asymptotic error bound:
	\begin{small}\begin{align*} &\quad\Expect[\widetilde{\gR}_{\mu}(\widehat{\pi},\epsilon)-\widetilde{\gR}_{\mu}(\widehat{\pi}^{*},\epsilon)] \\
  &\leq \frac{C_{1}R_{\max}}{(1-\lambda)^{\zeta/2}}\sqrt{d}d_K^{s+(\zeta\lor 1+1)/2}|\gD|^{\frac{-\zeta^*}{d_K+2\zeta^*}}\log(|\gD|)^{2+\frac{1}{\eta}}+C_{2}\sqrt{\epsilon},
	\end{align*}\end{small}
	where $0<\lambda<1$, $d_K=\gO(\mathrm{dim}_{\gM}(K)/\lambda^2)$, $\zeta^*=\zeta(1\land\zeta)$, $C_{1}$ is a constant depending on $s,B,\gC(\widehat{\pi};\mu),\gC(\widehat{\pi}_{\delta}^{*};\mu)$ and $C_{2}$ is a constant depending on $\gC(\widehat{\pi};\mu),\gC(\widehat{\pi}_{\delta}^{*};\mu)$.
\end{thm}

\noindent
When $d$ is large, this upper bound is less susceptible to the curse of dimensionality compared to the bound in Theorem \ref{thm:excess_risk}, since the intrinsic dimension $\mathrm{dim}_{\gM}(K)\ll d$. Additionally, the width of neural networks has a smaller order than that in Theorem \ref{thm:excess_risk}. These comparisons indicate a significant improvement in alleviating the curse of dimensionality.

\subsubsection{Low-Complexity Target Function.} We further consider a function $f$ combining $k$ functions as:
\begin{small}\begin{equation} \label{eq:low_complex}
f=G^k\circ G^{k-1}\circ \cdots \circ G^1,
\end{equation}\end{small}
where $G^i:\sR^{l_{i-1}}\to\sR^{l_{i}}$ is defined by
$G^{i}(x)=[g_1^{i}(W_1^{i}x),\ldots,g_{l_i}^{i}(W_{l_i}^{i}x)]^{\top}$, with $W_{j}^i\in\sR^{d_{i}\times l_{i-1}}$ being a matrix and $g_j^i:\sR^{d_i}\to\sR$ being a function.

For instance, a naive additive model \cite{stone1985additive} is given by $f(x)=g^1_1(x_1,x_2)+g^1_2(x_3,x_4)+\cdots+g^1_{2^{d-1}}(x_{2^d-1},x_{2^d})$ with $x=(x_1,\ldots,x_{2^d})^{\top}\in\sR^{2^d}$. In this case, it indicates $l_0=2^d$, $d_1=2$, $l_1=2^{d-1}$, $d_2=2^{d-1}$ and $l_2=1$. Without loss of generality, we assume that the component functions $g_j^i$ are \holder smooth with respect to the coefficient $\zeta_i$. 

\begin{rem}
    The low-complexity structure described in (\ref{eq:low_complex}) is common  in various models. Besides the additive model, similar structures can also be observed in other statistical inference models, e.g., the single index model \cite{hardle1993optimal}, the projection pursuit model \cite{friedman1981projection}. Moreover, recent research \cite{chen2023deep} has shown that operators associated with well-known PDEs, including the Poisson, parabolic, and Burgers equation, exhibit this structure or its variants. These operators have the potential to represent natural images for RL tasks such as CT scans \cite{shen2022learning}.
\end{rem}
\begin{thm} \label{thm:low_complex}
	Suppose the policy and value functions satisfy the condition in (\ref{eq:low_complex}).
	Assuming Assumptions \ref{asmp:holder},\ref{asmp:complete} and \ref{asmp:mixing} hold. Then for $\gN\gN_1$, $\gN\gN_2$ with width $\gW=\gO(d_*^{s+1}|\gD|^{\frac{d_*}{2d_*+4\zeta^*}})$ and depth $\gL=\gO(\log(|\gD|))$, we can establish a non-asymptotic error bound\:
\begin{small}
	\begin{align*}
	&\quad\Expect[\widetilde{\gR}_{\mu}(\widehat{\pi},\epsilon)-\widetilde{\gR}_{\mu}(\widehat{\pi}^{*},\epsilon)] \\
 &\leq 
	C_{1}R_{\max}d_*^{s+(\zeta\lor 1)/2}|\gD|^{\frac{-\zeta^*}{d_*+2\zeta^*}}\log(|\gD|)^{2+\frac{1}{\eta}}+C_{2}\sqrt{\epsilon},
	\end{align*}\end{small}
	where $\zeta^*=\min_i(\zeta_i\prod_{l=i+1}^{k}(\zeta^{l}\land 1))(1\land\zeta)$, $d_*=\max_i d_i$, $C_{1}$ is a constant depending on $B,\zeta,s,k,\gC(\widehat{\pi};\mu),\gC(\widehat{\pi}_{\delta}^{*};\mu)$ and $C_{2}$ is a constant depending on $\gC(\widehat{\pi};\mu),\gC(\widehat{\pi}_{\delta}^{*};\mu)$. 
\end{thm}
\noindent In Theorem \ref{thm:low_complex}, the expression for $\zeta^*$ differs from that in Theorems \ref{thm:excess_risk} and \ref{thm:low_dim}, as it is determined by the product of \holder coefficients of each component function $g_j^i$. A higher degree of smoothness in each component implies a tighter bound. The definition $d_*=\max_i d_i$ indicates a significant reduction in the curse of dimensionality since $d_i$ is consistently much smaller than $d$. The changes in $\zeta^*$ and $d_*$ also result in a reduced order for the width of the neural network.

Our results can be extended to anisotropic Besov spaces \cite{suzuki2018adaptivity,suzuki2021deep}, which aligns with our motivation related to low-complexity structure.

\section{Proof Sketch}
\label{sec:proof}

This section provides a proof sketch for Theorem \ref{thm:excess_risk}. The complete proof is available in Appendix B.

\subsection{Useful Lemmas}
\begin{lem} \label{ExcessRiskDCP}
Let $\widetilde{\gR}_{\mu}(\pi,\epsilon)$ be defined in ($\ref{defn:Rtilde}$), and let $\widehat{\pi},\widehat{\pi}^{*}$ be defined in (\ref{minimax}). For any admissible policy $\phi\in\Pi_{\theta}$, we have:
    \begin{small}\begin{align*}
&\quad\widetilde{\gR}_{\mu}(\widehat{\pi},\epsilon)-\widetilde{\gR}_{\mu}(\widehat{\pi}^{*},\epsilon) \\
 &\leq 
    2\underbrace{\sup_{\phi\in\Pi_{\theta}}|\widetilde{\gR}_{\gD}(\phi,\epsilon)-\widetilde{\gR}_{\mu}(\phi,\epsilon)|}_{(A)}
    + 2\underbrace{\sup_{\phi\in\Pi_{\theta}}|\widehat{\gR}_{\mu}(\phi,\epsilon)-\widetilde{\gR}_{\gD}(\phi,\epsilon)|}_{(B)} \\
    &\quad+\underbrace{\begin{aligned}&\inf_{\phi\in\Pi_{\theta}}\Big(\Big(\widehat{\gR}_{\mu}(\widehat{\pi},\epsilon)-\widehat{\gR}_{\gD}(\widehat{\pi},\epsilon)\Big) \\ &+\Big(\widehat{\gR}_{\gD}(\phi,\epsilon)-
    \widehat{\gR}_{\mu}(\phi,\epsilon)\Big)+
    \Big(\widetilde{\gR}_{\mu}(\phi,\epsilon)-\widetilde{\gR}_{\mu}(\widehat{\pi}^{*},\epsilon)\Big)\Big)\end{aligned}}_{(C)} 
    \end{align*}\end{small}
    where $\widetilde{\gR}_{\gD}(\pi,\epsilon)=\max_{f\in \gN\gN_{2}\cap \gF_{\mu}^{\pi,\epsilon}} \gR_{\mu}(\pi,f), \widehat{\gR}_{\mu}(\pi,\epsilon)=\max_{f\in \gN\gN_{2}\cap \gF_{\mu}^{\pi,\epsilon}} \gR_{\gD}(\pi,f)$ and $\widehat{\gR}_{\gD}(\pi,\epsilon)=\max_{f\in \gN\gN_{2}\cap \gF_{\gD}^{\pi,\epsilon}} \gR_{\gD}(\pi,f)$.
\end{lem}
\noindent The upper bound for the excess risk has three components. The first term (A) and the second term (B) capture the approximation error and generalization error, respectively. The third term (C), known as the Bellman estimation error, combines both the approximation and generalization error coupled with the Bellman residual about on/off-support data. This lemma provides a decomposition of the excess risk, forming the foundation for further derivation. We now introduce lemmas corresponding to each part.
\begin{lem} [Bounding (A)] \label{lem:approxbound}
Let $\widetilde{\gR}_{\gD}(\pi,\epsilon), \widetilde{\gR}_{\mu}(\pi,\epsilon)$ be defined in Lemma \ref{ExcessRiskDCP}. Under Assumption \ref{asmp:holder}, for $\gN\gN_{2}$ with width $38(s+1)^{2}3^{d}d^{s+1}N\ceil{\log_{2}(8N)}$ and depth $21(s+1)^{2}M\ceil{\log_{2}(8M)}+2d$ , it holds for any $M,N\in\sN^{+}$:
    \begin{small}\begin{equation*}
         \sup_{\phi\in\Pi_{\theta}}|\widetilde{\gR}_{\gD}(\phi,\epsilon)-\widetilde{\gR}_{\mu}(\phi,\epsilon)|\leq 38B(s+1)^{2}d^{s+(\zeta\lor 1)/2} (NM)^{-2\zeta/d}
    \end{equation*}\end{small}
\end{lem}
\noindent Lemma \ref{lem:approxbound} provides an approximation error between ReLU FNN and \holder functions, exhibiting a polynomial dependency on the input data dimension $d$.

\begin{lem} [Bounding (B)] \label{lem:genbound}
Let $\widehat{\gR}_{\mu}(\pi,\epsilon)$ and $\widetilde{\gR}_{\gD}(\pi,\epsilon)$ be defined as in Lemma \ref{ExcessRiskDCP}. Under Assumption \ref{asmp:mixing}, if the size of the dataset $|\gD|$ satisfies
\begin{small}\begin{align*}
|\gD|\geq n_{0}:=\max&\Big\{
\min\Big\{m\geq 3: \\
&m^{2}\geq 808c,\frac{m}{(\log m)^{2/\eta}}\Big\},e^{3/b}\Big\},
\end{align*}\end{small}
where $b,c,\eta$ are parameters in Assumption \ref{asmp:mixing}, the following holds for any $\phi\in\Pi_{\theta}$
    \begin{small}\begin{align*}
     &\quad\Expect\sup_{\phi\in\Pi_{\theta}}|\widehat{\gR}_{\mu}(\phi,\epsilon)-\widetilde{\gR}_{\gD}(\phi,\epsilon)| \\
     &\leq \gO\Big(R_{\max}\sqrt{\gP\gL\log(\gP)}\frac{\big(\log|\gD|\big)^{\frac{2+\eta}{2\eta}}}{\sqrt{|\gD|}}\Big).
    \end{align*}\end{small}
\end{lem}
\noindent Lemma \ref{lem:genbound} provides a bound on the generalization error for $\gC$-mixing data. We employ a uniform covering, to enable the measure of the infinite neural network class.

\begin{lem}[Bouding (C)] \label{lem:BellmanEstimation}
Let $\widehat{\gR}_{\mu}(\pi,\epsilon)$, $\widehat{\gR}_{\gD}(\pi,\epsilon)$, $\widehat{\gR}_{\mu}(\pi,\epsilon)$ be defined in  Lemma \ref{ExcessRiskDCP}, and let $\widehat{\pi},\widehat{\pi}^{*}$ be defined in  (\ref{minimax}). If the size of the dataset $|\gD|$ satisfies the requirement in Lemma \ref{lem:genbound}, we have
    \begin{small}\begin{align*} 	   &\quad\inf_{\phi\in\Pi_{\theta}}\Big(\Big(\widehat{\gR}_{\mu}(\widehat{\pi},\epsilon)-\widehat{\gR}_{\gD}(\widehat{\pi},\epsilon)\Big)+\Big(\widehat{\gR}_{\gD}(\phi,\epsilon)-
    \widehat{\gR}_{\mu}(\phi,\epsilon)\Big) \\
    &\quad\quad+\Big(\widetilde{\gR}_{\mu}(\phi,\epsilon)-\widetilde{\gR}_{\mu}(\widehat{\pi}^{*},\epsilon)\Big)\Big) \\
   & \leq C_{\gC(\widehat{\pi};\mu),\gC(\widehat{\pi}_{\delta}^{*};\mu)}\sqrt{\epsilon}+C_{B,\gC(\widehat{\pi}_{\delta}^{*};\mu)}\delta^{1\land \zeta} \\
   &\quad\quad+C_{\gC(\widehat{\pi};\mu),\gC(\widehat{\pi}_{\delta}^{*};\mu)}\gO\Big(R_{\max}\sqrt{\gP\gL\log(\gP)}\frac{\big(\log|\gD|\big)^{\frac{2+\eta}{2\eta}}}{\sqrt{|\gD|}}\Big).
	\end{align*}\end{small}
\end{lem}
\noindent The Bellman estimation consists of two terms: Bellman approximation and generalization. We bound these two terms together for consistency since a shared policy $\phi$ is considered an infimum. The first term on the RHS is related to the Bellman residual constraint, while the second and third terms correspond to generalization and approximation. The coefficients are similar to those discussed in Theorem \ref{thm:excess_risk}.

\subsection{Main Proof}
\textit{Proof of Theorem \ref{thm:excess_risk}.}
Combining the bounds in Lemma \ref{lem:approxbound}, \ref{lem:genbound}, \ref{lem:BellmanEstimation} into the decomposition in Lemma \ref{ExcessRiskDCP} yields
\begin{small}\begin{align*}
&\quad\Expect[\widetilde{\gR}_{\mu}(\widehat{\pi},\epsilon)-\widetilde{\gR}_{\mu}(\widehat{\pi}^{*},\epsilon)]  \\
&\leq C_{s,B,\gC(\widehat{\pi}_{\delta}^{*};\mu)}d^{s+(\zeta\lor 1)/2} (NM)^{-2\zeta(1\land \zeta)/d}+C_{\gC(\widehat{\pi};\mu),\gC(\widehat{\pi}_{\delta}^{*};\mu)}\sqrt{\epsilon} \\
&\quad +C_{\gC(\widehat{\pi};\mu),\gC(\widehat{\pi}_{\delta}^{*};\mu)}\Big(R_{\max}\sqrt{\gP\gL\log(\gP)}\frac{\big(\log|\gD|\big)^{\frac{2+\eta}{2\eta}}}{\sqrt{|\gD|}}\Big).
\end{align*}\end{small}
This bound indicates that as $M$ and $N$ grow large, the first term decreases while the second term increases. Thus we balance them by selecting appropriate $M$ and $N$ to obtain
\begin{small}\[ d^{s+(\zeta\lor 1)/2} (NM)^{-2\zeta(1\land \zeta)/d}\approx \sqrt{\gP\gL\log(\gP)}\frac{\big(\log|\gD|\big)^{\frac{2+\eta}{2\eta}}}{\sqrt{|\gD|}}. \]\end{small}
The number of parameters $\gP$, the width $\gW$ and the depth $\gL$ of the network satisfy the inequality:
\begin{small}\[ \gP\leq\gW(d+1)+(\gW^2+\gW)(\gL-1)+\gW+1\leq2\gW^2\gL. \]\end{small}
The approximation bound is established with width $\gW=38(s+1)^{2}3^{d}d^{s+1}N\ceil{\log_{2}(8N)}$ and depth $\gL=21(s+1)^{2}M\ceil{\log_{2}(8M)}+2d$, yielding the number of parameters $\gP\leq \gO((s+1)^6 d^{2s+2}N^2\ceil{\log_{2}(8N)}^2M\ceil{\log_{2}(8M)})$. By setting $N=\gO(|\gD|^{\frac{d}{2d+4\zeta^*}})$ and $M=\gO(\log(|\gD|))$, we can further bound $\gW=\gO(d^{s+1}|\gD|^{\frac{d}{2d+4\zeta^*}}),\gL=\gO(\log(|\gD|)),\gP=\gO(d^{2s+2}|\gD|^{\frac{d}{d+2\zeta^*}}\log(|\gD|))$.
\begin{small}\begin{align*}
&\quad\Expect[\widetilde{\gR}_{\mu}(\widehat{\pi},\epsilon)-\widetilde{\gR}_{\mu}(\widehat{\pi}^{*},\epsilon)]  \\
&\leq C_{s,B,\gC(\widehat{\pi}_{\delta}^{*};\mu)}d^{s+(\zeta\lor 1)/2} |\gD|^{\frac{-\zeta^*}{d+2\zeta^*}}\log(|\gD|)+C_{\gC(\widehat{\pi};\mu),\gC(\widehat{\pi}_{\delta}^{*};\mu)}\sqrt{\epsilon} \\
&\quad +C_{\gC(\widehat{\pi};\mu),\gC(\widehat{\pi}_{\delta}^{*};\mu)}R_{\max}d^{s+1}|\gD|^{\frac{-\zeta^*}{d+2\zeta^*}}\log(|\gD|)^{2+\frac{1}{\eta}} \\
&=C_{1}R_{\max}d^{s+(\zeta\lor 1)/2}|\gD|^{\frac{-\zeta^*}{d+2\zeta^*}}\log(|\gD|)^{2+\frac{1}{\eta}}+C_{2}\sqrt{\epsilon},
\end{align*}\end{small}
where $\zeta^*=\zeta(1\land\zeta)$, $C_{1}$ depends on $s,B,\gC(\widehat{\pi};\mu),\gC(\widehat{\pi}_{\delta}^{*};\mu)$, and $C_{2}$ depends on $\gC(\widehat{\pi};\mu)$ and $\gC(\widehat{\pi}_{\delta}^{*};\mu)$. \hfill\qedsymbol
    
\section{Conclusion}
\label{sec:conclusion}

This paper examines the estimation error within a deep adversarial offline RL framework under mild assumptions. Both policy and value functions are parameterized using deep neural networks, with data assumed to exhibit dependence and partial coverage. The excess risk is decomposed into three components: generalization, approximation, and Bellman estimation error. We bound these errors by adapting tools from empirical processes and approximation theory to address intricate Bellman constraints. This derived bound explicitly reveals the interplay between network architecture, dataset dimensionality, sample size, and the concentrability of distributional shifts in influencing the estimation error. Additionally, we provide two conditions to alleviate the curse of dimensionality. Our work is the first attempt to establish a non-asymptotic estimation error for deep adversarial offline RL problems.

\newpage
\section*{Acknowledgement}
We would like to thank the anonymous referees for their useful comments and suggestions, which have led to considerable improvements in the paper. This work is supported by the National Key Research and Development Program of China (No.2020YFA0714200), the National Nature Science Foundation of China (No.12371424, No.12371441), “the Fundamental Research Funds for the Central Universities”, the research fund of KLATASDSMOE of China, and the US National Science Foundation under awards DMS2244988, DMS2206333.

\bibliography{aaai24}

\newpage
\onecolumn
\appendix
\part*{Appendix}

This appendix is structured as follows:
\begin{itemize}
    \item In Section \ref{apdx:notation_table}, we provide an overview of the notations we use throughout the paper.
    \item In Section \ref{apdx:proof_excessrisk}, we provide the complete proof of Theorem \ref{thm:excess_risk}, followed by a flow chart of the derivation.
    \item In Section \ref{apdx:proof_CoD}, we provide the complete proof of Section \ref{sec:curse_of_dim}.
\end{itemize}

\section{Table of Notation} \label{apdx:notation_table}
Table \ref{tab:notations} summarizes the notations used in this paper.

\begin{table}[h]
    \centering
    \vspace{-0.15cm}
    \begin{tabular}{c|l}
    \toprule
    \textbf{Notation} & \textbf{Description} \\
    \midrule
    $\sR$ & Real number set \\
    $\sN^{+}$ & Positive integer set \\
    $\sN_{0}$ & Non-negative integer set \\
    $\rho^{\pi}$ & The marginal state-action occupancy measure with respect to $\pi$ \\
    $\mu$ & The behavior policy or $\rho^{\mu}$ \\
    $x$ & Input of a data point \\
    $d$ & Data dimension, i.e., $x\in\sR^d$ \\
    $A\cap B$ & The intersection of two sets $A$ and $B$ \\
    $A\subset B$ & A set $A$ is a subset of a set $B$ \\
    $\sigma(x)$ & ReLU activation function \\
    $\mathcal{N}\mathcal{N}$ & ReLU FNNs \\
    $\mathcal{W}$ & Network width \\
    $\mathcal{L}$ & Network depth \\
    $\mathcal{P}$ & The total number of network parameters \\
    $\|x\|_{q}$ & The $\ell_{q}$ norm ($q\geq 1$) of a vector $x\in\sR^d$ defined as $\|x\|_{q}:=(\sum_{i=1}^d |x_{i}|)^{1/q}$ \\ 
    $\|x\|_{\infty}$ & The $\ell_{\infty}$ norm of a vector $x\in\sR^d$ defined as $\|x\|_{\infty}:=\max(|x_1|,\ldots,|x_d|)$ \\
    $\Omega$ & A sample space \\
    $\gF$ & A $\sigma$-algebra on the set $\Omega$ \\
    $(\Omega,\gF)$ & A measurable space, i.e., a set $Z$ and a $\sigma$-algebra $\gB$ on $Z$ \\
    $ \mu $ & A measure on $(\Omega, \gF)$ \\
    $(\Omega, \gF, \mu)$ & A measure space \\
    $\gL_{p}(\Omega, \gF, \mu)$ & The subset of measurable functions $f$ defined on $\Omega$ that satisfy $\int_{\Omega}|f|^p d\mu < \infty$ \\
    $\mathrm{cor}(\cdot,\cdot)$ & Correlation of two random variables \\
    $\|f\|_{p,\pi}$ & $\pi$-weighted $L_{p}$-norm of a measurable function $f$, defined as $\|f\|_{p,\pi}:=(\Expect_{x\sim \pi}[|f(x)|^p])^{1/p}$ \\
    $\|f\|_{\infty}$ & Uniform norm, defined as $\|f\|_{\infty}:=\sup\{|f(x)|:x\text{ belongs to the domain of the function}\}$ \\
    $\|f\|$ & Seminorm of function $f$ \\
    $\floor{a}$ & The largest integer smaller than $a$ \\
    $\ceil{a}$ & The smallest integer larger than $a$ \\
    $\gD$ & Dataset \\
    $|\gD|$ & The cardinality of a set $\gD$ \\
    $\theta$ & Policy network model parameters \\
    $\Expect_{\gD}[f]$ & The empirical average of a function $f$ over a set $\gD$, defined as $\Expect_{\gD}[f]:=\frac{1}{|\gD|}\sum_{x\in\gD}f(x)$ \\
    $\gO(g(x))$ & $f(x)=\gO(g(x))$ indicates there exists $M>0$ and $x_0\in\sR$ s.t., $|f(x)|\leq Mg(x)$ for all $x\geq x_0$ \\
    $o(g(x))$ & $f(x)=o(g(x))$ indicates for $\forall \epsilon>0$, there exists a constant $x_0$ s.t., $|f(x)|\leq \epsilon g(x)$ for all $x\geq x_0$ \\
    $\dist$ & A metric \\
    $\gN(K,\dist,\epsilon)$ & Covering number of K, i.e., the smallest cardinality of an $\epsilon$-net of K \\
    $B_{E}$ & A closed unit ball of a space $E$ \\
    $\partial^{\alpha}$ & A differential operator vector, defined as $\partial^{\alpha}:=(\partial^{\alpha_1}\cdots\partial^{\alpha_d})$ with $\alpha=(\alpha_1,\ldots,\alpha_d)\in\sN_0^{d}$ \\
    $a\land b$ & $\min\{a,b\}$ \\
    $a\lor b$ & $\max\{a,b\}$ \\
    $f\circ g$ & Composite function \\
    \bottomrule
    \end{tabular}
    \caption{A summary of the notations used in this paper.}
    \label{tab:notations}
    \vspace{-0.4cm}
\end{table}

\section{Proofs of Theorem \ref{thm:excess_risk}}
\label{apdx:proof_excessrisk}
This section presents the proof of Theorem \ref{thm:excess_risk}, along with the proofs of several lemmas. These lemmas are discussed in Section \ref{sec:results}. The following figure summarizes the procedure of our proof.
\begin{figure}[h]
  \centering
  \includegraphics[width=0.95\textwidth]{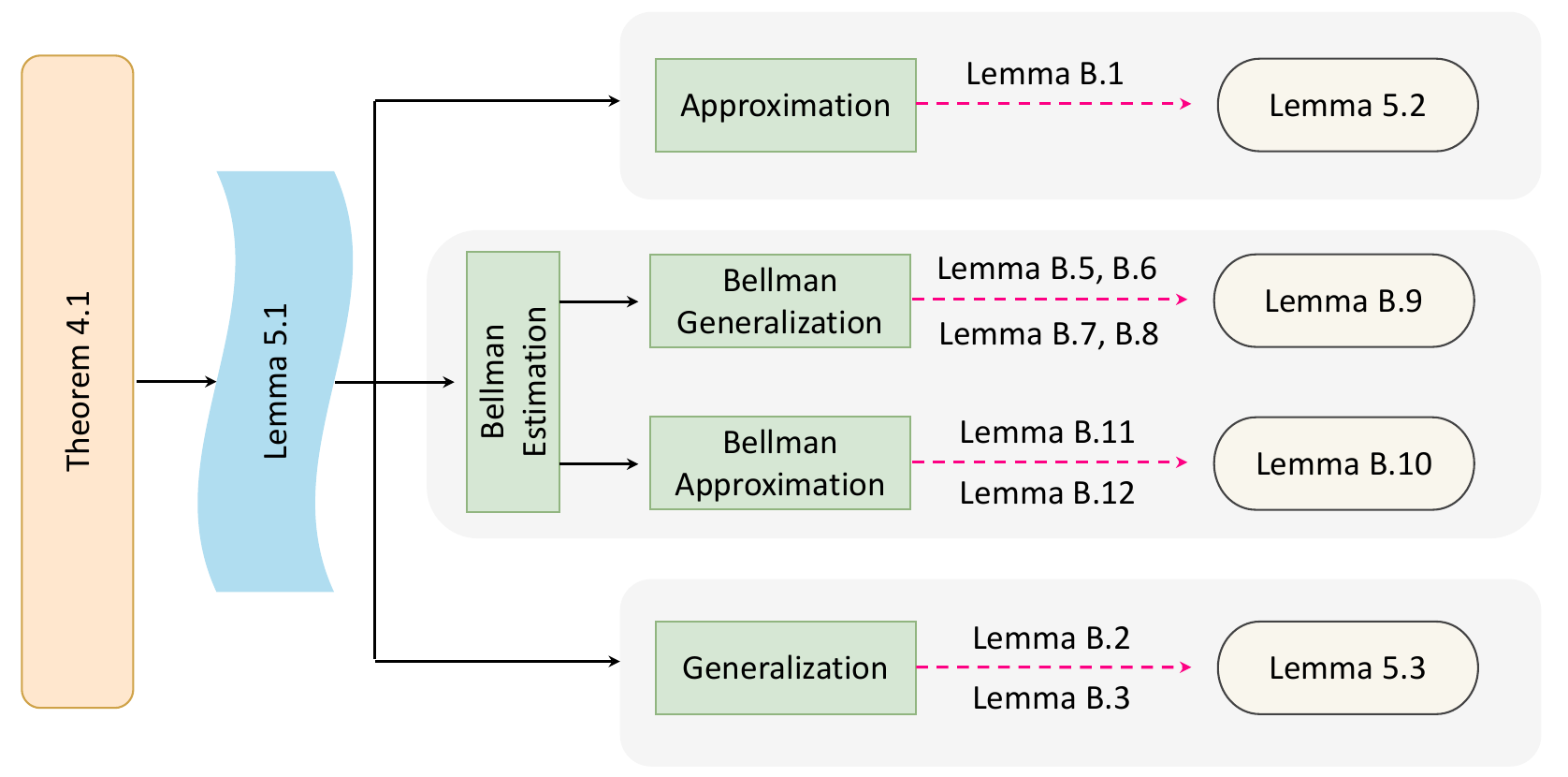}
  \caption{A flow chart of the proofs in Section B}
\end{figure}


\subsection{Proof of Lemma \ref{ExcessRiskDCP}}

We initiate the analysis by decomposing $\widetilde{\gR}_{\mu}(\widehat{\pi},\epsilon)-\widetilde{\gR}_{\mu}(\widehat{\pi}^{*},\epsilon)$ for any $\phi\in\Pi_{\theta}$:
\begin{align} \label{erdcp}
&\quad\quad\widetilde{\gR}_{\mu}(\widehat{\pi},\epsilon)-\widetilde{\gR}_{\mu}(\widehat{\pi}^{*},\epsilon) \notag \\
&=
\Big(\widetilde{\gR}_{\mu}(\widehat{\pi},\epsilon)-\widetilde{\gR}_{\gD}(\widehat{\pi},\epsilon)\Big)
+\Big(\widetilde{\gR}_{\gD}(\widehat{\pi},\epsilon)-\widehat{\gR}_{\mu}(\widehat{\pi},\epsilon)\Big)
+\Big(\widehat{\gR}_{\mu}(\widehat{\pi},\epsilon)-\widehat{\gR}_{\gD}(\widehat{\pi},\epsilon)\Big) \notag \\
&\quad+\Big(\widehat{\gR}_{\gD}(\widehat{\pi},\epsilon)-\widehat{\gR}_{\gD}(\phi,\epsilon)\Big)
+\Big(\widehat{\gR}_{\gD}(\phi,\epsilon)-
\widehat{\gR}_{\mu}(\phi,\epsilon)\Big)
+\Big(\widehat{\gR}_{\mu}(\phi,\epsilon)-\widetilde{\gR}_{\gD}(\phi,\epsilon)\Big) \notag \\
&\quad+\Big(\widetilde{\gR}_{\gD}(\phi,\epsilon)-\widetilde{\gR}_{\mu}(\phi,\epsilon)\Big)
+\Big(\widetilde{\gR}_{\mu}(\phi,\epsilon)-\widetilde{\gR}_{\mu}(\widehat{\pi}^{*},\epsilon)\Big)
\end{align}

The first and seventh terms are bounded by $\sup_{\phi\in\Pi_{\theta}}|\widetilde{\gR}_{\gD}(\phi,\epsilon)-\widetilde{\gR}_{\mu}(\phi,\epsilon)|$, quantifying the approximation capability for $\gN\gN_{2}\cap \gF_{\mu}^{\pi,\epsilon}$ to approximate $\gF_{\mu}^{\pi,\epsilon}$. The second and sixth terms are bounded by $\sup_{\phi\in\Pi_{\theta}}|\widehat{\gR}_{\mu}(\phi,\epsilon)-\widetilde{\gR}_{\gD}(\phi,\epsilon)|$, assessing the generalization behavior between $\gD$ and $\mu$ within $\gN\gN_{2}\cap \gF_{\mu}^{\pi,\epsilon}$. The third and fifth terms quantify the discrepancy between $\gN\gN_{2}\cap \gF_{\mu}^{\pi,\epsilon}$ and $\gN\gN_{2}\cap \gF_{\gD}^{\pi,\epsilon}$, representing a generalization error. The fourth term is negative according to the optimality of $\widehat{\pi}$. By taking an infimum over (\ref{erdcp}) we obtain
\begin{align*}
&\widetilde{\gR}_{\mu}(\widehat{\pi},\epsilon)-\widetilde{\gR}_{\mu}(\widehat{\pi}^{*},\epsilon) \leq
2\sup_{\phi\in\Pi_{\theta}}|\widetilde{\gR}_{\gD}(\phi,\epsilon)-\widetilde{\gR}_{\mu}(\phi,\epsilon)| + 2\sup_{\phi\in\Pi_{\theta}}|\widehat{\gR}_{\mu}(\phi,\epsilon)-\widetilde{\gR}_{\gD}(\phi,\epsilon)| \\
&+\inf_{\phi\in\Pi_{\theta}}\left(\Big(\widehat{\gR}_{\mu}(\widehat{\pi},\epsilon)-\widehat{\gR}_{\gD}(\widehat{\pi},\epsilon)\Big)+\Big(\widehat{\gR}_{\gD}(\phi,\epsilon)-
\widehat{\gR}_{\mu}(\phi,\epsilon)\Big)
+\Big(\widetilde{\gR}_{\mu}(\phi,\epsilon)-\widetilde{\gR}_{\mu}(\widehat{\pi}^{*},\epsilon)\Big)\right).
\end{align*}
This gives the decomposition in Lemma \ref{ExcessRiskDCP}.

\subsection{Proof of Lemma \ref{lem:approxbound}}
The proof of Lemma \ref{lem:approxbound} depends on the following auxiliary results:
\begin{lem}[Corollary 3.1 in \cite{jiao2023deep}] \label{approxbound}
    Assuming that $f$ belongs to $\gH^{\zeta}$ with $\zeta=s+r$, where $s$ is a positive integer and $0 < r \leq 1$, there exists a ReLU feedforward neural network $\widetilde{f}(x)$ with width $38(s+1)^{2}3^{d}d^{s+1}N\ceil{\log_{2}(8N)}$ and depth $21(s+1)^{2}M\ceil{\log_{2}(8M)}+2d$ for any $M,N\in\sN^{+}$, such that the following statement is true for any $x\in[0,1]^{d}$ :
    \[
        |f(x)-\widetilde{f}(x)|\leq 19B(s+1)^{2}d^{s+(\zeta\lor 1)/2} (NM)^{-2\zeta/d}.
    \]
\end{lem}

To compare the disparities between the regularization terms $\widetilde{\gR}_{\gD}(\phi,\epsilon)$ and $\widetilde{\gR}{\mu}(\phi,\epsilon)$, we can reframe the problem as a comparison between a \holder function class and a neural network class. More precisely, the former corresponds to a class of functions that adhere to a specific smoothness constraint, whereas the latter encompasses functions that can be represented by a neural network with a predetermined number of hidden layers and nodes.
\begin{align*}
&\quad\sup_{\phi\in\Pi_{\theta}}|\widetilde{\gR}_{\gD}(\phi,\epsilon)-\widetilde{\gR}_{\mu}(\phi,\epsilon)| \\
&=\sup_{\phi\in\Pi_{\theta}}\left|\max_{f\in \gN\gN_{2}\cap \gF_{\mu}^{\pi,\epsilon}} \gR_{\mu}(\pi,f)-\max_{f\in \gF_{\mu}^{\pi,\epsilon}} \gR_{\mu}(\pi,f)\right| \\
&=\sup_{\phi\in\Pi_{\theta}}\left|\max_{f\in \gN\gN_{2}\cap \gF_{\mu}^{\pi,\epsilon}} \Expect_{\mu}\left[f(s,a)-f(s,\phi)\right]-\max_{f\in \gF_{\mu}^{\pi,\epsilon}}\Expect_{\mu}\left[f(s,a)-f(s,\phi)\right]\right| \\
&=\sup_{\phi\in\Pi_{\theta}}\left(\max_{g\in \gF_{\mu}^{\pi,\epsilon}}\Expect_{\mu}\left[g(s,a)-g(s,\phi)\right]-\max_{f\in \gN\gN_{2}\cap \gF_{\mu}^{\pi,\epsilon}} \Expect_{\mu}\left[f(s,a)-f(s,\phi)\right]\right)
\end{align*}
Define $\delta=\sup_{g\in \gF_{\mu}^{\pi,\epsilon}}\inf_{f\in \gN\gN_{2}\cap \gF_{\mu}^{\pi,\epsilon}}|f-g|$. Then for given 
\[ g*=\mathop{\arg\max}_{g\in \gF_{\mu}^{\pi,\epsilon}}\sup_{\phi\in\Pi_{\theta}}\Expect_{\mu}\left[g(s,a)-g(s,\phi)\right], \] there exists $f_1\in \gF_{\mu}^{\pi,\epsilon}$ such that $|f_1-g^*|\leq \delta$. We obtain
\begin{align} 
&\quad\sup_{\phi\in\Pi_{\theta}}\left(\max_{g\in \gF_{\mu}^{\pi,\epsilon}}\Expect_{\mu}\left[g(s,a)-g(s,\phi)\right]-\max_{f\in \gN\gN_{2}\cap \gF_{\mu}^{\pi,\epsilon}} \Expect_{\mu}\left[f(s,a)-f(s,\phi)\right]\right) \nonumber \\
&\leq\sup_{\phi\in\Pi_{\theta}}\left(\Expect_{\mu}\left[g^*(s,a)-g^*(s,\phi)\right]-\Expect_{\mu}\left[f_1(s,a)-f_1(s,\phi)\right]\right) \nonumber \\
&\leq 2\delta =
2\sup_{g\in \gF_{\mu}^{\pi,\epsilon}}\inf_{f\in \gN\gN_{2}\cap \gF_{\mu}^{\pi,\epsilon}}|f-g|. \label{approxterm}
\end{align}

Extensive research in the literature has focused on the approximation theory of neural networks and smooth functions \cite{yarotsky2017error, lu2021deep, jiao2023approximation}. Researchers have explored the expressive power of neural networks, their capacity to represent specific function classes, and the error bounds associated with their approximations. By leveraging the crucial insights provided in Lemma \ref{approxbound}, we can derive a more precise bound on the difference in (\ref{approxterm}):
\[
\sup_{\phi\in\Pi_{\theta}}|\widetilde{\gR}_{\gD}(\phi,\epsilon)-\widetilde{\gR}_{\mu}(\phi,\epsilon)|\leq 2\sup_{g\in\gH^{\zeta}}\inf_{f\in\gN\gN_{2}}|f-g|\leq
38B(s+1)^{2}d^{s+(\zeta\lor 1)/2} (NM)^{-2\zeta/d}
\]
if we choose $\gN\gN_2$ to be a ReLU FNN with width $38(s+1)^{2}3^{d}d^{s+1}N\ceil{\log_{2}(8N)}$ and depth $21(s+1)^{2}M\ceil{\log_{2}(8M)}+2d$.

\subsection{Proof of Lemma \ref{lem:genbound}} \label{GenErrorSection}
We start with the expansion of $|\widehat{\gR}_{\mu}(\phi,\epsilon)-\widetilde{\gR}_{\gD}(\phi,\epsilon)|$:
    \begin{align*}
	&\quad\sup_{\phi\in\Pi_{\theta}}|\widehat{\gR}_{\mu}(\phi,\epsilon)-\widetilde{\gR}_{\gD}(\phi,\epsilon)| \\
	&=
	\sup_{\phi\in\Pi_{\theta}}\left|\max_{f\in \gN\gN_{2}\cap \gF_{\mu}^{\phi,\epsilon}} \gR_{\gD}(\phi,f)-\max_{f\in \gN\gN_{2}\cap \gF_{\mu}^{\phi,\epsilon}} \gR_{\mu}(\phi,f)\right| \\
	&\leq
	\sup_{\phi\in\Pi_{\theta},f\in \gN\gN_{2}\cap \gF_{\mu}^{\phi,\epsilon}}\left|\frac{1}{|\gD|} \sum_{(s,a)\in\gD}\Big(f(s,a)-f(s,\phi)\Big)-\Expect_{\mu}[f(s,a)-f(s,\phi)]\right|
\end{align*}
To simplify the notation, we define $\ell_{\phi,f}(s,a):=f(s,a)-f(s,\phi)$ and the corresponding function class as $\ell(\Pi_{\theta}\circ\gN\gN_{2}\cap \gF_{\mu}^{\phi,\epsilon})$. Subsequently, the inequality mentioned above can be expressed as:
{ \scriptsize
\begin{align*}
	&\sup_{\phi\in\Pi_{\theta}}|\widehat{\gR}_{\mu}(\phi,\epsilon)-\widetilde{\gR}_{\gD}(\phi,\epsilon)| \\
	\leq&
	\sup_{\phi\in\Pi_{\theta},f\in \gN\gN_{2}\cap \gF_{\mu}^{\phi,\epsilon}}\left|\frac{1}{|\gD|} \sum_{(s,a)\in\gD}\ell_{\phi,f}(s,a)-\Expect_{\mu}[\ell_{\phi,f}(s,a)]\right| \\
	\leq&
	\beta+\int_{\beta}^{4V_{\max}}\sP\left(\sup_{\phi\in\Pi_{\theta},f\in \gN\gN_{2}\cap \gF_{\mu}^{\phi,\epsilon}}\left|\frac{1}{|\gD|} \sum_{(s,a)\in\gD}\ell_{\phi,f}(s,a)-\Expect_{\mu}[\ell_{\phi,f}(s,a)]\right|>\epsilon\right)d\epsilon \\
	\leq&
	\beta+\int_{\beta}^{4V_{\max}}\sP\left(\sup_{\phi\in\Pi_{\theta},f\in \gN\gN_{2}\cap \gF_{\mu}^{\phi,\epsilon}}\left|\frac{1}{|\gD|} \sum_{(s,a)\in\gD}\ell_{\phi,f}(s,a)-\Expect_{\mu}[\ell_{\phi,f}(s_{0},a_{0})]\right|>\epsilon\right)d\epsilon \\
	\leq&
	\beta+\int_{\beta}^{4V_{\max}}4\gN_{\infty}\Big(\epsilon/4,\ell(\Pi_{\theta}\circ\gN\gN_{2}\cap \gF_{\mu}^{\phi,\epsilon}),|\gD|\Big)\exp\left(-\frac{3|\gD|\epsilon^{2}}{\big(\log|\gD|\big)^{2/\eta}\big(384\cdot(2V_{\max})^{2}+64\epsilon\cdot 2V_{\max}\big)}\right)d\epsilon.
\end{align*}}
The second inequality utilizes the upper bound of $\ell_{\phi,f}$ and divides the interval into two parts. The third inequality leverages the stationary property, while the fourth inequality relies on Lemma \ref{cct}. To derive a more precise bound, it is necessary to analyze the covering number of the composite function class. It is worth noting that we employ the concept of uniform covering number \cite{anthony1999neural}, denoted as
\[
	\gN_{\infty}(\epsilon,\gF,k)=\max\{\gN(\epsilon,\gF_{|x},d_{\infty}):x\in X^{k}\}
\]
and $\ell(\Pi_{\theta}\circ\gN\gN_{2}\cap \gF_{\mu}^{\phi,\epsilon}):=\left\{f(s,a)-f(s,\phi(s))\Big|(s,a)\in\gD,f\in\gN\gN_{2}\cap \gF_{\mu}^{\phi,\epsilon},\phi\in\Pi_{\theta}\right\}$
\begin{lem}\label{cvnb}
The function class $\ell(\Pi_{\theta}\circ\gN\gN_{2}\cap \gF_{\mu}^{\phi,\epsilon})$ is associated with a uniform covering number that can be upper-bounded as follows:
\[\gN_{\infty}\Big(\epsilon,\ell(\Pi_{\theta}\circ\gN\gN_{2}\cap \gF_{\mu}^{\phi,\epsilon}),|\gD|\Big)
\leq		\gN_{\infty}\Big(\epsilon/2,\gN\gN_{2},|\gD|\Big)\Big)^2
\]
\end{lem}
\begin{proof}[Proof of Lemma \ref{cvnb}]
	Let $\gD'$ be a set composed of $(s,\phi(s))$. We begin by selecting an $\epsilon/2$ uniform cover of $\gN\gN_{2}$ restricted to $\gD$, denoted as $E_{1}$. Similarly, we choose an $\epsilon/2$ uniform cover of $\gN\gN_{2}$ restricted to $\gD'$, denoted as $E_{2}$. Let $f(s,a)$ and $f(s,\phi(s))$ be arbitrary functions. We select $y_{1}\in E_{1}$ and $y_{2}\in E_{2}$ such that their distances from $f(s,a)$ and $f(s,\phi(s))$ are less than or equal to $\epsilon/2$, respectively. Therefore,
	\[
		\| f(s,a)-f(s,\phi(s))-(y_{1}-y_{2}) \|_{\infty}
		\leq
		\|f(s,a)-y_{1}\|_{\infty}+\|f(s,\phi(s))-y_{2}\|_{\infty}
		\leq \epsilon.
	\]
	Consequently, $E_{3}:=\{y_{1}-y_{2}:y_{1}\in E_{1},y_{2}\in E_{2}\}$ forms an $\epsilon$ cover of $\ell(\Pi_{\theta}\circ\gN\gN_{2}\cap \gF_{\mu}^{\phi,\epsilon})$. Hence, we can state that
\[ \gN_{\infty}\Big(\epsilon,\ell(\Pi_{\theta}\circ\gN\gN_{2}\cap \gF_{\mu}^{\phi,\epsilon}),|\gD|\Big)
\leq		\gN_{\infty}\Big(\epsilon/2,\gN\gN_{2},|\gD|\Big)\cdot\gN_{\infty}\Big(\epsilon/2,\gN\gN_{2},|\gD'|\Big). \]
We obtain this inequality by considering the covering numbers of the uniform covers, and since $\gD$ and $\gD'$ have the same size, we can conclude that the desired result holds.
\end{proof}
Then we can obtain an upper bound of $\sup_{\phi\in\Pi_{\theta}}|\widehat{\gR}_{\mu}(\phi,\epsilon)-\widetilde{\gR}_{\gD}(\phi,\epsilon)|$:
{\scriptsize
\begin{align*}
	&\Expect\sup_{\phi\in\Pi_{\theta}}|\widehat{\gR}_{\mu}(\phi,\epsilon)-\widetilde{\gR}_{\gD}(\phi,\epsilon)| \\
	\leq&
	\beta+\int_{\beta}^{4V_{\max}}4\gN_{\infty}\Big(\epsilon/2,\gN\gN_{2},|\gD|\Big)^2\exp\left(-\frac{3|\gD|\epsilon^{2}}{\big(\log|\gD|\big)^{2/\eta}\big(384\cdot(2V_{\max})^{2}+64\epsilon\cdot 2V_{\max}\big)}\right)d\epsilon \\
	\leq&
	\beta+\int_{\beta}^{4V_{\max}}4\left(\frac{eV_{\max}|\gD|}{\epsilon/2\cdot \VC(\gN\gN_{2})}\right)^{2\VC(\gN\gN_{2})}\exp\left(-\frac{3|\gD|\epsilon^{2}}{\big(\log|\gD|\big)^{2/\eta}\big(1536V_{\max}^{2}+64\cdot 4V_{\max}\cdot 2V_{\max}\big)}\right)d\epsilon \\
	\leq&
	\beta+16V_{\max}\left(\frac{eV_{\max}|\gD|}{\beta/2\cdot \VC(\gN\gN_{2})}\right)^{2\VC(\gN\gN_{2})}\exp\left(-\frac{3|\gD|\beta^{2}}{2048V_{\max}^{2}\big(\log|\gD|\big)^{2/\eta}}\right) \\
	=&\beta+16V_{\max}\exp\left(2\VC(\gN\gN_{2})\log\frac{2eV_{\max}|\gD|}{\beta\cdot \VC(\gN\gN_{2})}-\frac{3|\gD|\beta^{2}}{2048V_{\max}^{2}\big(\log|\gD|\big)^{2/\eta}}\right) \\
	=&\beta+16V_{\max}\exp\left(2\VC(\gN\gN_{2})\log\frac{2eV_{\max}|\gD|}{\VC(\gN\gN_{2})}-\frac{3|\gD|\beta^{2}}{2048V_{\max}^{2}\big(\log|\gD|\big)^{2/\eta}}-2\VC(\gN\gN_{2})\log\beta\right),
\end{align*} }
where the second inequality is justified by the relationship between the VC-dimension and the covering number \cite{anthony1999neural}. By setting 
\[
    \beta^{2}=\gO\left(\frac{V_{\max}^{2}\big(\log|\gD|\big)^{2/\eta}}{|\gD|}\VC(\gN\gN_{2})\log\frac{V_{\max}|\gD|}{\VC(\gN\gN_{2})}\right),
\]
we derive a final bound for the generalization error with $\VC(\gN\gN_{2})=\gO(\gP\gL\log(\gP))$ \cite{bartlett2019nearly} as follows:
\[
    \Expect\sup_{\phi\in\Pi_{\theta}}|\widehat{\gR}_{\mu}(\phi,\epsilon)-\widetilde{\gR}_{\gD}(\phi,\epsilon)| \leq C_{R_{\max}}\left(\sqrt{\gP\gL\log(\gP)}\frac{\big(\log|\gD|\big)^{\frac{2+\eta}{2\eta}}}{\sqrt{|\gD|}}\right),
\]
where $C_{R_{\max}}=\gO(R_{\max})$.

\begin{lem}[Theorem 3.1 in \cite{hang2017bernstein}] \label{cct}
	Assume that $(X_{i})_{i\in[n]}$ is a stochastic process on $(\Omega,\gA,\sP)$ which is strictly stationary and exponentially decayed $\gC$-mixing with $b,c,\eta>0$. Then for any $\epsilon>0$ and a measurable function class $\gF_{B}$ bounded by $B$, there exists $n_{0}:=\max\Big\{\min\Big\{m\geq 3:m^{2}\geq 808c,\frac{m}{(\log m)^{2/\eta}}\Big\},e^{3/b}\Big\}$ such that for any $n\geq n_{0}$, the following holds
 {\small
	\[
		\sP\left(\sup_{f\in\gF_{B}}\left|\frac{1}{n}\sum_{t=i}^{n}f(X_{i})-\Expect[f(X_{1})]\right|>\epsilon\right)
		\leq
		4\gN\left(\frac{\epsilon}{4},\gF_{B},\|\cdot\|_{\infty}\right)\exp\left(-\frac{3n\epsilon^{2}}{(\log n)^{2/\eta}(384B^{2}+64\epsilon B)}\right).
	\] }
\end{lem}

\subsection{Proof of Lemma \ref{lem:BellmanEstimation}}

We begin by extending the classical Bellman contraction property in the $L_{\infty}$-norm to the $L_1$-norm.

\begin{lem} \label{contract}
	Let $\rho^{\phi}$ represent the stationary distribution of the policy $\phi$. For any measurable function $f$ defined on $\gS\times\gA$ and $p\geq 1$, the following inequality holds:
	\[ \|f-f^*\|_{1,\mu}\leq \frac{\|f-\gT^{\phi}f\|_{p,\mu}}{1-\gamma\gC(\phi;\mu)}. \]
\end{lem}

\begin{rem}
	It is worth noting that this result is similar to the well-known observation $\|V-V^{\pi}\|_{\infty}\leq \frac{\|V-\gT^{\pi}V\|_{\infty}}{1-\gamma}$, where the contraction property in Lemma \ref{contract} is stated in the $L_p$ form with respect to the stationary distribution $\rho^{\phi}$.
\end{rem}

\begin{proof}[Proof of Lemma \ref{contract}] 
	First, we perform a decomposition as follows:
	\begin{align*}
	f-f^*&=f-\gT^{\phi}f+\gT^{\phi}f-f^* \\
	&=(f-\gT^{\phi}f)+(\gT^{\phi}f-\gT^{\phi}f^*).
	\end{align*}
	Next, we can evaluate the second term at $(s,a)$.
	\begin{align*}
	\gT^{\phi}(s,a)-\gT^{\phi}f^*(s,a)&=\gamma\int\gP(ds'|s,a)\phi(da'|s')f(s',a')-\gamma\int\gP(ds'|s,a)\phi(da'|s')f^*(s',a') \\
	&=\gamma\int\gP(ds'|s,a)\phi(da'|s')\left(f(s',a')-f^*(s',a')\right).
	\end{align*}
	Now, we calculate the absolute integral as follows:
{ \small
	\begin{align*}
	&\quad\int\left|f(s,a)-f^*(s,a)\right| d\mu(s,a) \\
	&\leq \int\left|f(s,a)-\gT^{\phi} f(s,a)\right|d\mu(s,a)+\int\left|\gT^{\phi} f(s,a)-\gT^{\phi}f^*(s,a)\right|d\mu(s,a) \\
	&=\int\left|f(s,a)-\gT^{\phi} f(s,a)\right|d\mu(s,a)+\int\left|\gamma\int\gP(ds'|s,a)\phi(da'|s')\left(f(s',a')-f^*(s',a')\right)\right|d\mu(s,a) \\
	&\leq\int\left|f(s,a)-\gT^{\phi} f(s,a)\right|d\mu(s,a)+\gamma\int d\mu(s,a)\gP(ds'|s,a)\phi(da'|s')\left|f(s',a')-f^*(s',a')\right| \\
	&\leq\int\left|f(s,a)-\gT^{\phi} f(s,a)\right|d\mu(s,a)+\gamma\gC(\phi;\mu)\int d\mu(s',a')\left|f(s',a')-f^*(s',a')\right|.
	\end{align*} }
Here, we use the Jensen inequality in the second inequality, and the last equality is due to the stationary distribution and the concentration property. Thus, we obtain the inequality:
	\[ \|f-f^*\|_{1,\mu}\leq \|f-\gT^{\phi} f\|_{1,\mu}+\gamma\gC(\phi;\mu)\|f-f^*\|_{1,\mu}. \]
	By rearranging the above inequality and applying Jensen's inequality, we have:
	\[ \|f-f^*\|_{1,\mu}\leq\frac{\|f-\gT^{\phi} f\|_{1,\mu}}{1-\gamma\gC(\phi;\mu)}\leq\frac{\|f-\gT^{\phi} f\|_{2,\mu}}{1-\gamma\gC(\phi;\mu)} \]
\end{proof}

\paragraph{Bellman Generalization}
We first address the first two terms on the left-hand side (LHS) of Lemma \ref{lem:BellmanEstimation}. The proof commences with a decomposition:
\begin{align}
&\quad\left|\widehat{\gR}_{\mu}(\phi,\epsilon)-\widehat{\gR}_{\gD}(\phi,\epsilon)\right| \notag \\
&=\left|\gL_{\gD}(\phi,f_{\gD})-\gL_{\gD}(\phi,f_{\mu})\right| \nonumber \\
&=\left|\gL_{\gD}(\phi,f_{\gD})-\gL_{\mu}(\phi,f_{\gD})+\gL_{\mu}(\phi,f_{\gD})-\gL_{\mu}(\phi,f_{\mu})+\gL_{\mu}(\phi,f_{\mu})-\gL_{\gD}(\phi,f_{\mu})\right| \nonumber \\
&\leq \left|\gL_{\gD}(\phi,f_{\gD})-\gL_{\mu}(\phi,f_{\gD})\right|+\left|\gL_{\mu}(\phi,f_{\gD})-\gL_{\mu}(\phi,f_{\mu})\right|+\left|\gL_{\mu}(\phi,f_{\mu})-\gL_{\gD}(\phi,f_{\mu})\right|. \label{BellmanGenErrorDCP},
\end{align}
where for simplicity, we introduce the notations $f_{\gD}:=\argmin_{f\in\gN\gN_{2}\cap \gF_{\gD}^{\phi,\epsilon}}\Expect_{\gD}[f(s,\phi)-f(s,a)]$ and $f_{\mu}:=\argmin_{f\in\gN\gN_{2}\cap \gF_{\mu}^{\phi,\epsilon}}\Expect_{\gD}[f(s,\phi)-f(s,a)]$. Before proceeding to bound the Bellman generalization error, we present several auxiliary lemmas:
\begin{lem}[Performance Difference Lemma] \label{PDL}
Let $\mu$ and $\pi$ be two distinct policies, and let $J(\mu):=\Expect[\sum_{t=0}^{\infty}\gamma^t r_t|a_t\sim\pi(\cdot|s_t)]$ represent the cumulative return. In this regard, the following holds:
	\[
	J(\mu)-J(\pi)=\frac{1}{1-\gamma}\Expect_{\mu}\left[Q^{\pi}(s,a)-Q^{\pi}(s,\pi)\right].
	\]
\end{lem}

\begin{proof}[Proof of Lemma \ref{PDL}]
	This proof is based on the work by \citet{kakade2002approximately}.
	\begin{align*}
	&\quad J(\mu)-J(\pi) \\
	&=\Expect_{s_t\sim \mu,a_t\sim \mu(\cdot|s_t)}\left[\sum_{t=0}^{\infty}\gamma^{t}r(s_t,a_t)\right]-J(\pi) \\
	&=\Expect_{s_t\sim \mu,a_t\sim \mu(\cdot|s_t)}\left[\sum_{t=0}^{\infty}\gamma^{t}\left(r(s_t,a_t)+Q^{\pi}(s_t,\pi)-Q^{\pi}(s_t,\pi)\right)\right]-Q^{\pi}(s_0,\pi) \\
	&=\Expect_{s_t\sim \mu,a_t\sim \mu(\cdot|s_t)}\left[\sum_{t=0}^{\infty}\gamma^{t}\left(r(s_t,a_t)+\gamma Q^{\pi}(s_{t+1},\pi)-Q^{\pi}(s_t,\pi)\right)\right]+Q^{\pi}(s_0,\pi)-Q^{\pi}(s_0,\pi) \\
	&=\Expect_{s_t\sim \mu,a_t\sim \mu(\cdot|s_t)}\left[\sum_{t=0}^{\infty}\gamma^{t}\left(Q^{\pi}(s_t,a_t)-Q^{\pi}(s_t,\pi)\right)\right] \\
	&=\frac{1}{1-\gamma}\Expect_{\mu}\left[Q^{\pi}(s,a)-Q^{\pi}(s,\pi)\right].
	\end{align*}
\end{proof}

\begin{lem}[General Performance Difference Lemma] \label{genpd}
	Let $\pi$ and $\pi'$ be two distinct policies, and let $s_0$ be an initial state. In this context, we have the following:
	\begin{equation*}
	f(s_{0},\pi')-J(\pi)=\frac{1}{1-\gamma}\Expect_{s_t\sim d_{s_0}^{\pi},a_t\sim \pi(\cdot|s_t)}[f(s,\pi')-\gT^{\pi'}f(s,a)]
	\end{equation*}
\end{lem}
\begin{proof}[Proof for Lemma \ref{genpd}]
	This proof is based on the work by \citet{xie2020q}.
	\begin{align*}
	J(\pi)&=\Expect_{s_t\sim d_{s_0}^{\pi},a_t\sim \pi(\cdot|s_t)}\left[\sum_{t=0}^{\infty}\gamma^{t}r(s_t,a_t)\right] \\
	&=\Expect_{s_t\sim d_{s_0}^{\pi},a_t\sim \pi(\cdot|s_t)}\left[\sum_{t=0}^{\infty}\gamma^{t}\left(r(s_t,a_t)+f(s_t,\pi')-f(s_t,\pi')\right)\right] \\
	&=\Expect_{s_t\sim d_{s_0}^{\pi},a_t\sim \pi(\cdot|s_t)}\left[\sum_{t=0}^{\infty}\gamma^{t}\left(r(s_t,a_t)+\gamma f(s_{t+1},\pi')-f(s_t,\pi')\right)\right]+f(s_0,\pi') \\
	&=\Expect_{s_t\sim d_{s_0}^{\pi},a_t\sim \pi(\cdot|s_t)}\left[\sum_{t=0}^{\infty}\gamma^{t}\left(\gT^{\pi'}f(s_t,a_t)-f(s_t,\pi')\right)\right]+f(s_0,\pi') \\
	&=\frac{1}{1-\gamma}\Expect_{s_t\sim d_{s_0}^{\pi},a_t\sim \pi(\cdot|s_t)}\left[\gT^{\pi'}f(s,a)-f(s,\pi')\right]+f(s_0,\pi').
	\end{align*}
\end{proof}

By utilizing Lemma \ref{PDL} and \ref{genpd}, we can reformat $\gL_{\mu}(\pi,f_1)-\gL_{\mu}(\pi,f_2)$ to resemble the Bellman residual, employing the approach described in \citet{xie2020q}.

\begin{lem}\label{bgdcp}
Let $f_1$ and $f_2$ be two measurable functions defined on $\gS\times\gA$, and let $\gL_{\mu}(\pi,f)$ be defined in (\ref{maximin}). Then the following holds:
\begin{align*}
\gL_{\mu}(\pi,f_1)-\gL_{\mu}(\pi,f_2) &=\Expect_{\mu}[\gT^{\pi}f_1(s,a)-f_1(s,a)]+\Expect_{d_{\mu}^{\pi}}[f_1(s,a)-\gT^{\pi}f_1(s,a)] \\
&\quad +\Expect_{\mu}[f_2(s,a)-\gT^{\pi}f_2(s,a)]+\Expect_{d_{\mu}^{\pi}}[\gT^{\pi}f_2(s,a)-f_2(s,a)].
\end{align*}
\end{lem}

\begin{proof}[Proof of Lemma \ref{bgdcp}]
	\[
	\gL_{\mu}(\pi,f_1)-\gL_{\mu}(\pi,f_2)
	=\gL_{\mu}(\pi,f_1)-\gL_{\mu}(\pi,Q^{\pi})+\gL_{\mu}(\pi,Q^{\pi})-\gL_{\mu}(\pi,f_2),
	\]
	We begin by examining the first two terms:
	\begin{align*}
	\gL_{\mu}(\pi,Q^{\pi})&=\Expect_{\mu}\left[Q^{\pi}(s,\pi)-Q^{\pi}(s,a)\right] \\
	&=(1-\gamma)\left(J(\pi)-J(\mu)\right) \\
	&=(1-\gamma)\left(J(\pi)-f_1(s_0,\pi)\right)+(1-\gamma)\left(f_1(s_0,\pi)-J(\mu)\right) \\
	&=\Expect_{d_{\mu}^{\pi}}[\gT^{\pi}f_1(s,a)-f_1(s,\pi)]+\Expect_{\mu}[f_1(s,\pi)-\gT^{\pi}f_1(s,a)] \\
	&=\Expect_{d_{\mu}^{\pi}}[\gT^{\pi}f_1(s,a)-f_1(s,\pi)]+\Expect_{\mu}[f_1(s,\pi)-f_1(s,a)+f_1(s,a)-\gT^{\pi}f_1(s,a)] \\
	&=\Expect_{d_{\mu}^{\pi}}[\gT^{\pi}f_1(s,a)-f_1(s,a)]+\gL_{\mu}(\pi,f_1)+\Expect_{\mu}[f_1(s,a)-\gT^{\pi}f_1(s,a)].
	\end{align*}
	In the second equality, we utilize Lemma \ref{PDL}, and the fourth equality follows from the application of Lemma \ref{genpd}. Thus, we obtain:
	\begin{align*}
	\gL_{\mu}(\pi,f_1)-\gL_{\mu}(\pi,Q^{\pi})=\Expect_{\mu}[\gT^{\pi}f_1(s,a)-f_1(s,a)]+\Expect_{d_{\mu}^{\pi}}[f_1(s,a)-\gT^{\pi}f_1(s,a)] 
	\end{align*}
	Similarly, we obtain:
	\begin{align*}
	\gL_{\mu}(\pi,Q^{\pi})-\gL_{\mu}(\pi,f_2)=\Expect_{\mu}[f_2(s,a)-\gT^{\pi}f_2(s,a)]+\Expect_{d_{\mu}^{\pi}}[\gT^{\pi}f_2(s,a)-f_2(s,a)].
	\end{align*}
\end{proof}

\begin{lem} \label{BellmanTransferResidual}
	Suppose Assumption \ref{asmp:complete} exists, then for any $\pi\in\Pi_{\theta},f\in\gN\gN_2$, we have
	\begin{align*}
	\|f-\gT^{\pi}f\|_{2,\mu}\leq \gO\left(R_{\max}\sqrt{\gP\gL\log(\gP)}\frac{\big(\log|\gD|\big)^{\frac{2+\eta}{2\eta}}}{\sqrt{|\gD|}}\right)+\sqrt{\gE(f,\pi;\gD)}.
	\end{align*}
\end{lem}

The present lemma serves as an extension of Theorem 9 in \cite{cheng2022adversarially}. However, we enhance the conclusion by incorporating Dudley integration, as elaborated upon in Section \ref{GenErrorSection}. Herein, we proceed to establish the bound for the Bellman generalization error and provide a proof thereof.

\begin{lem}[Bellman Generalization]\label{lem:BellmanGenError}
Let $\widehat{\gR}_{\mu}(\pi,\epsilon)$ and $\widehat{\gR}_{\gD}(\pi,\epsilon)$ be defined as stated in Lemma \ref{ExcessRiskDCP}. The following inequality holds:
	\[
	\left|\widehat{\gR}_{\mu}(\pi,\epsilon)-\widehat{\gR}_{\gD}(\pi,\epsilon)\right| \leq
	(1+\gC(\pi;\mu))\left(2\sqrt{\epsilon}+\gO\left(\sqrt{\gP\gL\log(\gP)}\frac{\big(\log|\gD|\big)^{\frac{2+\eta}{2\eta}}}{\sqrt{|\gD|}}\right)\right)+2\epsilon_{gen}.
	\]
Here, $\epsilon_{gen}$ represents the error bound from Lemma \ref{lem:genbound}, and $\epsilon$ denotes the upper bound of the Bellman residual.
\end{lem}

\begin{proof}[Proof of Lemma \ref{lem:BellmanGenError}]
	We start by considering a decomposition of (\ref{BellmanGenErrorDCP})
	\begin{align*}
	\left|\widehat{\gR}_{\mu}(\pi,\epsilon)-\widehat{\gR}_{\gD}(\pi,\epsilon)\right| &\leq
	\left|\gL_{\gD}(\phi,f_{\gD})-\gL_{\mu}(\phi,f_{\gD})\right|+\left|\gL_{\mu}(\phi,f_{\gD})-\gL_{\mu}(\phi,f_{\mu})\right| \\
	&\quad+\left|\gL_{\mu}(\phi,f_{\mu})-\gL_{\gD}(\phi,f_{\mu})\right|.
	\end{align*}
	The first term can be bounded using an empirical process technique
	\begin{align*}
	&\quad\left|\gL_{\gD}(\phi,f_{\gD})-\gL_{\mu}(\phi,f_{\gD})\right| \\
	&=\left|\Expect_{\mu}[f_{\gD}(s,\phi)-f_{\gD}(s,a)]-\frac{1}{|\gD|}\sum_{(s,a)\in\gD}\left(f_{\gD}(s,\phi)-f_{\gD}(s,a)\right)\right| \\
	&\leq \sup_{f\in\gN\gN_2,\phi\in\Pi_{\theta}}\left|\Expect_{\mu}[f(s,\phi)-f(s,a)]-\frac{1}{|\gD|}\sum_{(s,a)\in\gD}\left(f(s,\phi)-f(s,a)\right)\right|.
	\end{align*}
	We observe that the generalization error analyzed in Section \ref{GenErrorSection} is precisely the issue at hand. To provide further clarity, we define $\ell(\Pi_{\theta}\circ\gN\gN_2)$ as the set of differences, that is, $\ell(\Pi_{\theta}\circ\gN\gN_2):=\left\{f(s,a)-f(s,\phi(s))|(s,a)\in\gD,\phi\in\Pi_{\theta},f\in\gN\gN_2\right\}$. Utilizing a result similar to Lemma \ref{cvnb}, we establish an equivalent bound, denoted as $\epsilon_{gen}$ in this section. Similarly, the third term can be addressed using the same approach and is also bounded by $\epsilon_{gen}$. In the remaining part of this proof, our focus will be on the second term.
	
	It is important to note that the Bellman operator is a contraction operator, ensuring the existence of a unique fixed point with respect to the policy, denoted as $f_{\pi}^*$. Consequently, we have the following decomposition of the second term:
	\begin{align*}
	&\quad \left|\gL_{\mu}(\pi,f_{\gD})-\gL_{\mu}(\pi,f_{\mu})\right| \\
	&=\big|\Expect_{\mu}[\gT^{\pi}f_{\gD}(s,a)-f_{\gD}(s,a)]+\Expect_{d_{\mu}^{\pi}}[f_{\gD}(s,a)-\gT^{\pi}f_{\gD}(s,a)] \\
	&\quad +\Expect_{\mu}[f_{\mu}(s,a)-\gT^{\pi}f_{\mu}(s,a)]+\Expect_{d_{\mu}^{\pi}}[\gT^{\pi}f_{\mu}(s,a)-f_{\mu}(s,a)]\big| \\
	&\leq (1+ \gC(\pi;\mu))\|f_{\gD}-\gT^{\pi}f_{\gD}\|_{1,\mu}+(1+ \gC(\pi;\mu))\|f_{\mu}-\gT^{\pi}f_{\mu}\|_{1,\mu} \\
	&\leq (1+ \gC(\pi;\mu))\left(2\sqrt{\epsilon}+\gO\left(\sqrt{\gP\gL\log(\gP)}\frac{\big(\log|\gD|\big)^{\frac{2+\eta}{2\eta}}}{\sqrt{|\gD|}}\right)\right).
	\end{align*}
	Here, we utilize Lemma \ref{bgdcp} for the first equality. For the final inequality, we employ the definitions of $f_{\gD}$ and $f_{\mu}$, as well as Lemma \ref{BellmanTransferResidual}.
\end{proof}

\paragraph{Bellman Approximation}

The final component of the error decomposition is associated with the error arising from the Bellman approximation, which is characterized by its intricate Bellman structure. The objective is to obtain an upper limit for this error.
\begin{align}
&\widetilde{\gR}_{\mu}(\phi,\epsilon)-\widetilde{\gR}_{\mu}(\widehat{\pi}^{*},\epsilon) \notag \\
=&\max_{f\in\gF_{\mu}^{\phi,\epsilon}}\Expect_{\mu}[f(s,a)-f(s,\phi)]-\max_{g\in\gF_{\mu}^{\widehat{\pi}^{*},\epsilon}}\Expect_{\mu}[g(s,a)-g(s,\widehat{\pi}^{*})] \notag \\
\leq & \max_{f\in\gF_{\mu}^{\phi,\epsilon},g\in\gF_{\mu}^{\widehat{\pi}^{*},\epsilon}}\Expect_{\mu}[f(s,a)-f(s,\phi)-g(s,a)+g(s,\widehat{\pi}^{*})] \notag \\
\leq & \max_{f\in\gF_{\mu}^{\phi,\epsilon},g\in\gF_{\mu}^{\widehat{\pi}^{*},\epsilon}}\Expect_{\mu}[f(s,a)-g(s,a)-f(s,\phi)+g(s,\phi)-g(s,\phi)+g(s,\widehat{\pi}^{*})] \notag \\
\leq & (1+\gC(\phi;\mu))\sup_{f\in\gF_{\mu}^{\phi,\epsilon},g\in\gF_{\mu}^{\widehat{\pi}^{*},\epsilon}}\|f-g\|_{1,\mu}+\max_{g\in\gF_{\mu}^{\widehat{\pi}^{*},\epsilon}}\Expect_{\mu}\left[g(s,\widehat{\pi}^{*})-g(s,\phi)\right] \label{badcp}
\end{align}

Utilizing the aforementioned decomposition, we present a lemma that offers a bound on the error resulting from the Bellman approximation.

\begin{lem} \label{lem:BellmanApproxError}
Let $\widetilde{\gR}_{\mu}(\phi,\epsilon)$ and $\widetilde{\gR}_{\mu}(\widehat{\pi}^{*},\epsilon)$ be defined as stated in Lemma \ref{ExcessRiskDCP}. The following inequality holds:
	\begin{align*}
	\widetilde{\gR}_{\mu}(\phi,\epsilon)-\widetilde{\gR}_{\mu}(\widehat{\pi}^{*},\epsilon)\leq & \frac{(1+\gC(\phi;\mu))\sqrt{\epsilon}}{1-\gamma\gC(\phi;\mu)} +\frac{(1+\gC(\phi;\mu))\sqrt{\epsilon}}{1-\gamma\gC(\widehat{\pi}^{*};\mu)}+\\
	&\quad\frac{(1+\gC(\phi;\mu))\|\phi-\widehat{\pi}^{*}\|_{1,\mu}}{(1-\gamma)^2-(1-\gamma)\|\phi-\widehat{\pi}^{*}\|_{1,\mu}}+B\|\widehat{\pi}^{*}-\phi\|_{\mu,\infty}^{1\land \zeta},
	\end{align*}
 where we denote $\|\widehat{\pi}^{*}-\phi\|_{\mu,\infty}:=\int_{a\in\gA}\Expect_{s\sim d^{\mu}}|\widehat{\pi}^{*}(da|s)-\phi(da|s)|$ with a little abuse of notation.
\end{lem}

\begin{proof}[Proof of Lemma \ref{lem:BellmanApproxError}]
	The proof is divided into two parts, corresponding to the two terms in equation (\ref{badcp}). It is possible to bound the second term by utilizing the continuity of $g$, whereas the first term necessitates a more detailed analysis. We begin with the second term.
	
	(PART I.) Based on the given definition, $g$ belongs to $\gH^{\zeta}$ with $\zeta=s+r$. In the case of $s=0$, i.e., $\zeta=r\in(0,1]$, we can derive $g(x)-g(y)\leq B\|x-y\|_{\infty}^{r}$ from the inequality $\sup_{x\neq y}\frac{g(x)-g(y)}{\|x-y\|_{\infty}^{r}}\leq B$. For $s\geq 1$, $\gH^{\zeta}$ is a subset of $\gH^{1}$, i.e., $f(x)-f(y)\leq B\|x-y\|_{\infty}$. By combining these two aspects, we present the bound for the second term as follows
	\begin{align*}
	\max_{g\in\gF_{\mu}^{\widehat{\pi}^{*},\epsilon}}\Expect_{\mu}\left[g(s,\widehat{\pi}^{*})-g(s,\phi)\right]&\leq
	B\|(s,\widehat{\pi}^{*}(s))-(s,\phi(s))\|_{\mu,\infty}^{1\land \zeta} \\
	&\leq B\|\widehat{\pi}^{*}-\phi\|_{\mu,\infty}^{1\land \zeta}.
	\end{align*}

(PART II.) Recall the definition of $\gF_{\pi,\epsilon}:=\{f\in\gF\mid\Expect_{\mu}[((f-\gT^{\pi}f)(s,a))^2]\le \epsilon,\gF\subseteq(\gS\times\gA\to[0,V_{\max}])\}$. Since the Bellman operator $\gT$ is a contraction mapping, a fixed point $f^*$ exists such that $\Expect_{\mu}[((f^*-\gT^{\pi}f^*)(s,a))^2]=0$. The fixed point serves as an intermediate function for bounding the first term, \ie
	\begin{align*}
	\|f-g\|_{1,\mu}\leq \|f-f^*\|_{1,\mu}+\|f^*-g^*\|_{1,\mu}+\|g-g^*\|_{1,\mu}.
	\end{align*}
	
	For both the first and third terms, Lemma \ref{contract} is employed to derive
	\begin{align*}
	\|f-f^*\|_{1,\mu}\leq \frac{\|f-\gT^{\phi}f\|_{2,\mu}}{1-\gamma\gC(\phi;\mu)}, \quad
	\|g-g^*\|_{1,\mu}\leq
	\frac{\|g-\gT^{\widehat{\pi}^{*}}g\|_{2,\mu}}{1-\gamma\gC(\widehat{\pi}^{*};\mu)}.
	\end{align*}
	Given that $f\in\gF_{\mu}^{\phi,\epsilon}$ and $g\in\gF_{\mu}^{\widehat{\pi}^{*},\epsilon}$, the following conclusion can be further derived:
	\begin{align*}
	\|f-f^*\|_{1,\mu}\leq \frac{\sqrt{\epsilon}}{1-\gamma\gC(\phi;\mu)}, \quad
	\|g-g^*\|_{1,\mu}\leq
	\frac{\sqrt{\epsilon}}{1-\gamma\gC(\widehat{\pi}^{*};\mu)}.
	\end{align*}
	
	The second term represents the drift between Bellman fixed points associated with different policies. By utilizing Lemma \ref{FixedPointDrift}, we can promptly obtain
	\begin{align*}
	\|f^*-g^*\|_{1,\mu}\leq  \frac{\|\phi-\widehat{\pi}^{*}\|_{1,\mu}}{(1-\gamma)^2-(1-\gamma)\|\phi-\widehat{\pi}^{*}\|_{1,\mu}}.
	\end{align*}
	
	By combining the aforementioned two parts, the following upper bound for the Bellman approximation error is obtained:
	\begin{align*}
	&\quad \widetilde{\gR}_{\mu}(\phi,\epsilon)-\widetilde{\gR}_{\mu}(\widehat{\pi}^{*},\epsilon) \\
	&\leq
	(1+\gC(\phi;\mu))\sup_{f\in\gF_{\mu}^{\phi,\epsilon},g\in\gF_{\mu}^{\widehat{\pi}^{*},\epsilon}}\|f-g\|_{1,\mu}+\max_{g\in\gF_{\mu}^{\widehat{\pi}^{*},\epsilon}}\Expect_{\mu}\left[g(s,\widehat{\pi}^{*})-g(s,\phi)\right] \\
	&\leq
	\frac{(1+\gC(\phi;\mu))\sqrt{\epsilon}}{1-\gamma\gC(\phi;\mu)}+\frac{(1+\gC(\phi;\mu))\sqrt{\epsilon}}{1-\gamma\gC(\widehat{\pi}^{*};\mu)}+\frac{(1+\gC(\phi;\mu))\|\phi-\widehat{\pi}^{*}\|_{1,\mu}}{(1-\gamma)^2-(1-\gamma)\|\phi-\widehat{\pi}^{*}\|_{1,\mu}}+B\|\widehat{\pi}^{*}-\phi\|_{\mu,\infty}^{1\land \zeta}.
	\end{align*}
\end{proof}

\begin{lem}\label{vexpression}
	Let $\phi$ be a stationary policy, and suppose that $f^*$ satisfies $f^*(s,a)=\gT^{\phi} f^*(s,a)$. It follows that:
	$$ f^*=(I-\gB^{\phi})^{-1}r, $$
	where $I$ denotes the identity mapping, and $\gB^{\phi}f(s,a):=\gamma\int\gP(ds'|s,a)\phi(da'|s')f(s',a')$ defines the operator $\gB^{\phi}$.
\end{lem}
\begin{proof}[Proof of Lemma \ref{vexpression}]
	The main objective is to demonstrate the invertibility of $I-\gB^{\phi}$ is. For any non-zero function $h$,
	\begin{align*}
	\|(I-\gB^{\phi})f\|_{\infty}&=\|f-\gB^{\phi}f\|_{\infty}\geq \|f\|_{\infty}-\|\gB^{\phi}f\|_{\infty} \geq \|f\|_{\infty}-\gamma\|f\|_{\infty} > 0.
	\end{align*}
\end{proof}
We can now utilize the following lemma to bound $\|f^*-g^*\|_{1,\mu}$
\begin{lem} \label{FixedPointDrift}
For the Bellman fixed points $f^*$ and $g^*$ corresponding to $\phi$ and $\widehat{\pi}^{*}$, the following inequality holds:
 \[ \|f^*-g^*\|_{1,\mu}\leq \frac{\|\phi-\widehat{\pi}^{*}\|_{1,\mu}}{(1-\gamma)^2-(1-\gamma)\|\phi-\widehat{\pi}^{*}\|_{1,\mu}} \]
\end{lem}
\begin{proof}
	First, the expressions of $f^*$ and $g^*$ are substituted using Lemma \ref{vexpression},
	\begin{align*}
	f^*-g^*=(I-\gB^{\phi})^{-1}r-(I-\gB^{\widehat{\pi}^{*}})^{-1}r.
	\end{align*}
	For simplicity, let $A=I-\gB^{\widehat{\pi}^{*}}$, yielding
	\begin{align*}
	f^*-g^*&=\left((I-\gB^{\widehat{\pi}^{*}})-(\gB^{\phi}-\gB^{\widehat{\pi}^{*}})\right)^{-1}r-A^{-1}r \\
	&= \left(A-(\gB^{\phi}-\gB^{\widehat{\pi}^{*}})\right)^{-1}r-A^{-1}r \\
	&=\left(A\left(I-A^{-1}(\gB^{\phi}-\gB^{\widehat{\pi}^{*}})\right)\right)^{-1}r-A^{-1}r \\
	&=\left(I-A^{-1}(\gB^{\phi}-\gB^{\widehat{\pi}^{*}})\right)^{-1}A^{-1}r-A^{-1}r.
	\end{align*}
	We then utilize the Neumann series expansion for the first term on the right-hand side, provided that $\|A^{-1}(\gB^{\phi}-\gB^{\widehat{\pi}^{*}})\|\leq 1$. This is justified by the boundedness of $\|A\|$ and the smallness of $\|\gB^{\phi}-\gB^{\widehat{\pi}^{*}}\|$ resulting from the approximation, ensuring the validity of using the Neumann series expansion. Therefore,
	\begin{align*}
	f^*-g^*&=\left(I+\sum_{n=1}^{\infty}\left(A^{-1}(\gB^{\phi}-\gB^{\widehat{\pi}^{*}})\right)^{n}\right)A^{-1}r-A^{-1}r \\
	&=A^{-1}(\gB^{\phi}-\gB^{\widehat{\pi}^{*}})A^{-1}r+\sum_{n=2}^{\infty}\left(A^{-1}(\gB^{\phi}-\gB^{\widehat{\pi}^{*}})\right)^{n}A^{-1}r.
	\end{align*}
	Our attention is now directed towards the operator norm.
	\begin{align*}
	\|f^*-g^*\|_{1,\mu}&\leq \left\|A^{-1}(\gB^{\phi}-\gB^{\widehat{\pi}^{*}})A^{-1}r\right\|_{1,\mu}+\left\|\sum_{n=2}^{\infty}\left(A^{-1}(\gB^{\phi}-\gB^{\widehat{\pi}^{*}})\right)^{n}A^{-1}r\right\|_{1,\mu}.
	\end{align*}
	We can bound the first term using the consistency of the norm as follows:
	\begin{align*}
	\left\|A^{-1}(\gB^{\phi}-\gB^{\widehat{\pi}^{*}})A^{-1}r\right\|_{1,\mu}\leq \left\|A^{-1}\right\|_{1,\mu}^2\left\|\gB^{\phi}-\gB^{\widehat{\pi}^{*}}\right\|_{1,\mu}\|r\|_{1,\mu}.
	\end{align*}
	Applying the Neumann series once more yields:
	\begin{align*}
	\left\|A^{-1}\right\|_{1,\mu}=\left\|I-\gB^{\widehat{\pi}^{*}}\right\|_{1,\mu}\leq\frac{1}{1-\|\gB^{\widehat{\pi}^{*}}\|_{1,\mu}}\leq \frac{1}{1-\gamma}.
	\end{align*}
	The last inequality follows from the definition of $\gB^{\widehat{\pi}^{*}}$ and the operator norm. Next, we analyze the approximation term, which can be expressed as the difference between $\phi$ and $\widehat{\pi}^{*}$.
	\begin{align*}
	\left(\gB^{\phi}-\gB^{\widehat{\pi}^{*}}\right)h(s,a)&=\int\gP(ds'|s,a)\phi(da'|s')h(s',a')-\int\gP(ds'|s,a)\widehat{\pi}^{*}(da'|s')h(s',a') \\
	&=\int\gP(ds'|s,a)\left(\phi(da'|s')-\widehat{\pi}^{*}(da'|s')\right)h(s',a') \\
	&\leq \int \left|h(s',a')\right|\cdot\gP(ds'|s,a)\cdot\left|\left(\phi(da'|s')-\widehat{\pi}^{*}(da'|s')\right)\right| \\
	&\leq \int P(ds'|s,a) \cdot \sup_{s'}\int \left|\left(\phi(da'|s')-\widehat{\pi}^{*}(da'|s')\right)\right| \\
	&\leq \|\phi-\widehat{\pi}^{*}\|_{1,\mu},
	\end{align*}
	Here, by using the upper bound of the value function and the property of density functions, with $\|h(s,a)\|_{1,\mu}=1$, we obtain the second and the last inequalities. Furthermore, according to the definition of the operator norm, we have:
	\begin{align*}
	\|\gB^{\phi}-\gB^{\widehat{\pi}^{*}}\|_{1,\mu}\leq V_{\max}\|\phi-\widehat{\pi}^{*}\|_{1,\mu}.
	\end{align*}
	Combining the above results, we finally obtain
	\begin{align*}
	&\quad\|f^*-g^*\|_{1,\mu} \\
	&\leq
	\left(\frac{1}{1-\gamma}\right)^2\cdot V_{\max}\cdot\|\phi-\widehat{\pi}^{*}\|_{1,\mu}\cdot\|r\|_{2,\mu}+\sum_{n=2}^{\infty}\left(\frac{1}{1-\gamma}\right)^{n+1}\cdot V_{\max}^{n}\cdot\|\phi-\widehat{\pi}^{*}\|_{1,\mu}^{n}\cdot\|r\|_{1,\mu} \\
	&\leq \frac{\|\phi-\widehat{\pi}^{*}\|_{1,\mu}}{(1-\gamma)^2-(1-\gamma)\|\phi-\widehat{\pi}^{*}\|_{1,\mu}}
	\end{align*}
\end{proof}

By utilizing Lemma \ref{lem:BellmanGenError} and Lemma \ref{lem:BellmanApproxError}, we are able to establish the proof of Lemma \ref{lem:BellmanEstimation}. We select $\phi$ as an approximation for $\widehat{\pi}^{*}$, serving as an upper bound for $\inf_{\phi\in\Pi_{\theta}}$. By combining the bounds presented in Lemma \ref{lem:BellmanGenError} and Lemma \ref{lem:BellmanApproxError}, we obtain
\begin{align*}
&\quad\inf_{\phi\in\Pi_{\theta}}\left(\Big(\widehat{\gR}_{\mu}(\widehat{\pi},\epsilon)-\widehat{\gR}_{\gD}(\widehat{\pi},\epsilon)\Big)+\Big(\widehat{\gR}_{\gD}(\phi,\epsilon)-
\widehat{\gR}_{\mu}(\phi,\epsilon)\Big)
+\Big(\widetilde{\gR}_{\mu}(\phi,\epsilon)-\widetilde{\gR}_{\mu}(\widehat{\pi}^{*},\epsilon)\Big)\right) \\
&\leq \inf_{\phi\in\Pi_{\theta}}
\left(1+ \gC(\widehat{\pi};\mu)+1+ \gC(\phi;\mu)\right)\left(2\sqrt{\epsilon}+\gO\left(\sqrt{\gP\gL\log(\gP)}\frac{\big(\log|\gD|\big)^{\frac{2+\eta}{2\eta}}}{\sqrt{|\gD|}}\right)\right)+4\epsilon_{gen} \\
&\quad+ \frac{(1+\gC(\phi;\mu))\sqrt{\epsilon}}{1-\gamma\gC(\widehat{\pi}^{*};\mu)}+
\frac{(1+\gC(\phi;\mu))\sqrt{\epsilon}}{1-\gamma\gC(\phi;\mu)}+\frac{(1+\gC(\phi;\mu))\|\phi-\widehat{\pi}^{*}\|_{1,\mu}}{(1-\gamma)^2-(1-\gamma)\|\phi-\widehat{\pi}^{*}\|_{1,\mu}}+B\|\widehat{\pi}^{*}-\phi\|_{\mu,\infty}^{1\land \zeta} \\
&\leq \left(2+ \gC(\widehat{\pi};\mu)+ \gC(\widehat{\pi}_{\delta}^{*};\mu)\right)\left(2\sqrt{\epsilon}+\gO\left(R_{\max}\sqrt{\gP\gL\log(\gP)}\frac{\big(\log|\gD|\big)^{\frac{2+\eta}{2\eta}}}{\sqrt{|\gD|}}\right)\right)+4\epsilon_{gen} \\
&\quad+ \frac{2(1+\gC(\widehat{\pi}_{\delta}^{*};\mu))\sqrt{\epsilon}}{1-\gamma\gC(\widehat{\pi}_{\delta}^{*};\mu)}++\frac{(1+\gC(\widehat{\pi}_{\delta}^{*};\mu))\|\phi-\widehat{\pi}^{*}\|_{1,\mu}}{(1-\gamma)^2-(1-\gamma)\|\phi-\widehat{\pi}^{*}\|_{1,\mu}}+B\delta^{1\land \zeta} \\
&=C_{\gC(\widehat{\pi};\mu),\gC(\widehat{\pi}_{\delta}^{*};\mu)}\sqrt{\epsilon}+C_{\gC(\widehat{\pi};\mu),\gC(\widehat{\pi}_{\delta}^{*};\mu)}\left(R_{\max}\sqrt{\gP\gL\log(\gP)}\frac{\big(\log|\gD|\big)^{\frac{2+\eta}{2\eta}}}{\sqrt{|\gD|}}\right)+C_{B,\gC(\widehat{\pi}_{\delta}^{*};\mu)}\delta^{1\land \zeta}.
\end{align*}
Here, $\delta$ represents the approximation error between the densities of $\phi$ and $\widehat{\pi}_{\delta}^{*}$, and $\widehat{\pi}_{\delta}^{*}$ is a $\delta$-neighborhood of $\widehat{\pi}^{*}$ under $L_{\infty}$-norm towards their densities. We define $\gC(\widehat{\pi}_{\delta}^{*};\mu)$ as the concentrability coefficient of $\mu$ with respect to the neighborhood $\widehat{\pi}_{\delta}^{*}$.

\section{Proof of Section \ref{sec:curse_of_dim}} \label{apdx:proof_CoD}
In this section, we present the complete proofs of conclusions that alleviate the curse of dimensionality.

\subsection{Proof of Theorem \ref{thm:low_dim}}
We initially introduce a lemma that addresses the scenario where the data is supported by a set with a low Minkowski dimension. This is established by Whitney's extension theorem and Lemma \ref{approxbound}.
\begin{lem} \label{lem:low_dim_approx}
	Let $K$ be a subset of $[0,1]^d$ with a Minkowski dimension denoted by $\mathrm{dim}_{\gM}(K)$. Assuming that $f$ belongs to $\gH^{\zeta}$ with $\zeta=s+r$, where $s$ is a positive integer and $0 < r \leq 1$, there exists a ReLU feedforward neural network $\widetilde{f}(x)$ with width $38(s+1)^{2}3^{d_{K}}d_{K}^{s+1}N\ceil{\log_{2}(8N)}$ and depth $21(s+1)^{2}M\ceil{\log_{2}(8M)}+2d_{K}$ for any positive integers $M$ and $N$, such that the following inequality holds for any $x\in K$ :
	\begin{align*}
	|f(x)-\widetilde{f}(x)|\leq C \frac{B}{(1-\lambda)^{\zeta/2}}(s+1)^{2}\sqrt{d}d_{K}^{s+(\zeta\lor 1+1)/2} (NM)^{-2\zeta/d_{K}}.
	\end{align*}
	Here $0<\lambda<1, d_K=\gO(\text{dim}_{\gM}(K)/\lambda^2)$ and $C$ represents a universal constant.
\end{lem}
By following the proof of Theorem \ref{thm:excess_risk}, we can obtain an excess risk bound that mitigates the curse of dimensionality.
\begin{align*}
&\quad\Expect[\widetilde{\gR}_{\mu}(\widehat{\pi},\epsilon)-\widetilde{\gR}_{\mu}(\widehat{\pi}^{*},\epsilon)]  \\
&\leq C_{s,B,\gC(\widehat{\pi}_{\delta}^{*};\mu)}\frac{1}{(1-\lambda)^{\zeta/2}}\sqrt{d}d_{K}^{s+(\zeta\lor 1+1)/2} (NM)^{-2\zeta^*/d_{K}}+ \\
&\quad C_{\gC(\widehat{\pi};\mu),\gC(\widehat{\pi}_{\delta}^{*};\mu)}\left(R_{\max}\sqrt{\gP\gL\log(\gP)}\frac{\big(\log|\gD|\big)^{\frac{2+\eta}{2\eta}}}{\sqrt{|\gD|}}\right)+C_{\gC(\widehat{\pi};\mu),\gC(\widehat{\pi}_{\delta}^{*};\mu)}\sqrt{\epsilon},
\end{align*}
Similarly, we set $N=\gO(|\gD|^{\frac{d_K}{2d_K+4\zeta^*}})$ and $M=\gO(\log(|\gD|))$. Consequently, we have $\gW=\gO(d_K^{s+1}|\gD|^{\frac{d_K}{2d_K+4\zeta}}),\gL=\gO(\log(|\gD|))$, and $\gP=\gO(d_K^{2s+2}|\gD|^{\frac{d_K}{d_K+2\zeta^*}}\log(|\gD|))$. Therefore,
\begin{align*}
&\quad\Expect[\widetilde{\gR}_{\mu}(\widehat{\pi},\epsilon)-\widetilde{\gR}_{\mu}(\widehat{\pi}^{*},\epsilon)]  \\
&\leq C_{s,B,\gC(\widehat{\pi}_{\delta}^{*};\mu)}\frac{1}{(1-\lambda)^{\zeta/2}}\sqrt{d}d_K^{s+(\zeta\lor 1+1)/2} |\gD|^{\frac{-\zeta}{d_K+2\zeta^*}}\log(|\gD|) \\
&\quad +C_{\gC(\widehat{\pi};\mu),\gC(\widehat{\pi}_{\delta}^{*};\mu)}R_{\max}d_K^{s+1}|\gD|^{\frac{-\zeta^*}{d_K+2\zeta^*}}\log(|\gD|)^{2+\frac{1}{\eta}}+C_{\gC(\widehat{\pi};\mu),\gC(\widehat{\pi}_{\delta}^{*};\mu)}\sqrt{\epsilon} \\
&\leq C_{s,B,\gC(\widehat{\pi};\mu),\gC(\widehat{\pi}_{\delta}^{*};\mu)}\frac{R_{\max}}{(1-\lambda)^{\zeta/2}}\sqrt{d}d_K^{s+(\zeta\lor 1+1)/2}|\gD|^{\frac{-\zeta^*}{d_K+2\zeta^*}}\log(|\gD|)^{2+\frac{1}{\eta}}+C_{\gC(\widehat{\pi};\mu),\gC(\widehat{\pi}_{\delta}^{*};\mu)}\sqrt{\epsilon} \\
&=\frac{C_{1}R_{\max}}{(1-\lambda)^{\zeta/2}}\sqrt{d}d_K^{s+(\zeta\lor 1+1)/2}|\gD|^{\frac{-\zeta^*}{d_K+2\zeta^*}}\log(|\gD|)^{2+\frac{1}{\eta}}+C_{2}\sqrt{\epsilon},
\end{align*}
where $\zeta^*=\zeta(1\land\zeta)$, $C_{1}$ is a constant depending on $s,B,\gC(\widehat{\pi};\mu),\gC(\widehat{\pi}_{\delta}^{*};\mu)$ and $C_{2}$ is a constant depending on $\gC(\widehat{\pi};\mu),\gC(\widehat{\pi}_{\delta}^{*};\mu)$.

\subsection{Proof of Theorem \ref{thm:low_complex}}
The primary objective is to establish a novel approximation bound for the low-complexity functions outlined in the subsequent lemma:
\begin{lem} \label{lem:low_complex_approx}
	Let f be a function defined in (\ref{eq:low_complex}). There exists a ReLU feedforward neural network $\widetilde{G}(x)$ with width $\max_{i}l_i\cdot38(s+1)^{2}3^{d_i}d_i^{s+1}N\ceil{\log_{2}(8N)}$ and depth $21k(s+1)^{2}M\ceil{\log_{2}(8M)}+2\sum_{i=1}^{k}d_i+3(k-1)$ for any $M,N\in\sN^{+}$, such that the following inequality holds for any $x\in[0,1]^{d}$
	\[ |f(x)-\widetilde{G}(x)|\leq C_{B,\zeta,s,k}d_*^{s+(\zeta\lor 1)/2} (NM)^{-2\zeta_k^*/d_*}, \]
	where $C_{B,\zeta,s,k}$ is a constant associated with $B,\zeta,s,k$, $d_*=\max_i d_i$, and $\zeta_k^*=\min_i\zeta_i\prod_{l=i+1}^{k}(\zeta^{l}\land 1)$.
\end{lem}

\begin{proof}[Proof of Lemma \ref{lem:low_complex_approx}] Recall the definition of function $f$:
	\begin{equation*}
	f=G^k\circ G^{k-1}\circ \cdots G^1,
	\end{equation*}
	where $G^i:\sR^{l_{i-1}}\to\sR^{l_{i}}$ is defined as
	\[ G^{i}(x)=[g_1^{i}(W_1^{i}x),\ldots,g_{l_i}^{i}(W_{l_i}^{i}x)], \]
	where $W_{j}^i\in\sR^{d_{i}\times l_{i-1}}$ is a matrix and $g_j^i:\sR^{d_i}\to\sR$ is a function.
	Without loss of generality, we assume that $g_j^i$ takes values in the interval $[0,1]$. First, we focus on the approximation error of each $g_j^i$. This error is a direct result of Lemma \ref{approxbound}. There exists a ReLU feedforward neural network $\widetilde{g}^i_j(x)$ with width $38(s+1)^{2}3^{d_i}d_i^{s+1}N\ceil{\log_{2}(8N)}$ and depth $21(s+1)^{2}M\ceil{\log_{2}(8M)}+2d_i$ for any $M,N\in\sN^{+}$, such that the following statement is true for any $x\in[0,1]^{d_i}$ :
	\begin{align*}
	|g_j^i(x)-\widetilde{g}^i_j(x)|\leq 19B(s+1)^{2}d_i^{s+(\zeta\lor 1)/2} (NM)^{-2\zeta_i/d_i}.
	\end{align*}
	Now we consider a neural network denoted as $\widetilde{g^i}$, which consists of $\{\widetilde{g}^i_j\}_{j=1,\ldots,l_i}$ in parallel. To ensure that $\{\widetilde{g}^i_j\}_{j=1,\ldots,l_i}$ take values in $[0,1]$, we apply a two-layer output transformation for $i<k$, represented as $\widetilde{G^i}$. The neural network $\widetilde{G^i}$ has a width of $l_i\cdot38(s+1)^{2}3^{d_i}d_i^{s+1}N\ceil{\log_{2}(8N)}$ and a depth of $21(s+1)^{2}M\ceil{\log_{2}(8M)}+2d_i+2$ for $i<k$. Finally, we combine the $\widetilde{G^i}$ as follows:
	\[ \widetilde{G}=\widetilde{G^k}\circ\widetilde{G^{k-1}}\circ\cdots\circ\widetilde{G^0}
	=\widetilde{g^k}\circ\sigma(\widetilde{g^{k-1}})\circ\cdots\circ\sigma(\widetilde{g^0}). \]
	The width of $\widetilde{G^i}$ is $l_i\cdot38(s+1)^{2}3^{d_i}d_i^{s+1}N\ceil{\log_{2}(8N)}$ and the depth is $21(s+1)^{2}M\ceil{\log_{2}(8M)}+2d_i+2$. Now we consider the difference between $f$ and $\widetilde{G}$. 
	\begin{align*}
	|f-\widetilde{G}|&=|G^k\circ G^{k-1:1}-\widetilde{G}^k\circ \widetilde{G}^{k-1:1}| \\
	&\leq|G^k\circ G^{k-1:1}-G^k\circ \widetilde{G}^{k-1:1}|+|G^k\circ \widetilde{G}^{k-1:1}-\widetilde{G}^k\circ \widetilde{G}^{k-1:1}| \\
	&\leq C_{B}|G^{k-1:1}-\widetilde{G}^{k-1:1}|^{\zeta_k\land 1}+|G^k-\widetilde{G}^k|,
	\end{align*}
	where $G^{k-1:1}=G^{k-1}\circ \cdots G^1$, and we utilize the triangle inequality together with the properties of \holder function class. By recursion, we obtain the following result:
	\[ |f-\widetilde{G}|\leq C_{B,\zeta}\sum_{i=1}^{k}|G^k-\widetilde{G}^k|^{\prod_{l=i+1}^{k}\zeta_{l}\land 1}. \]
	Since $\widetilde{G}$ is composed of $\widetilde{G^i}$, it is necessary to add a connecting layer between each component. Consequently, the resulting network has a width of $\max_{i}l_i\cdot38(s+1)^{2}3^{d_i}d_i^{s+1}N\ceil{\log_{2}(8N)}$, and a depth of $21k(s+1)^{2}M\ceil{\log_{2}(8M)}+2\sum_{i=1}^{k}d_i+3(k-1)$. The approximation error can be expressed as follows:
	\begin{align*}
	|f(x)-\widetilde{G}(x)|&\leq C_{B,\zeta}\sum_{i=1}^{k}|19B(s+1)^{2}d_i^{s+(\zeta\lor 1)/2} (NM)^{-2\zeta_i/d_i}|^{\prod_{l=i+1}^{k}(\zeta_{l}\land 1)} \\
	&\leq C_{B,\zeta,s}\sum_{i=1}^{k}d_i^{s+(\zeta\lor 1)/2} (NM)^{-2\zeta_i\prod_{l=i+1}^{k}(\zeta_{l}\land 1)/d_i} \\
	&\leq C_{B,\zeta,s,k}\max_{i}d_i^{s+(\zeta\lor 1)/2} (NM)^{-2\zeta_i\prod_{l=i+1}^{k}(\zeta_{l}\land 1)/d_i} \\
	&=C_{B,\zeta,s,k}d_*^{s+(\zeta\lor 1)/2} (NM)^{-2\zeta_k^*/d_*},
	\end{align*} 
	where $d_*=\max_i d_i$ and $\zeta_k^*=\min_i\zeta_i\prod_{l=i+1}^{k}(\zeta_{l}\land 1)$.
\end{proof}
Next, we provide the proof of Theorem \ref{thm:low_complex} by utilizing Lemma \ref{lem:low_complex_approx}. Following the proof of Theorem \ref{thm:excess_risk}, we obtain
\begin{align*}
&\quad\Expect[\widetilde{\gR}_{\mu}(\widehat{\pi},\epsilon)-\widetilde{\gR}_{\mu}(\widehat{\pi}^{*},\epsilon)]  \\
&\leq C_{B,\zeta,s,k,\gC(\widehat{\pi}_{\delta}^{*};\mu)}d_*^{s+(\zeta\lor 1)/2} (NM)^{-2\zeta_k^*(1\land \zeta)/d}+ \\
&\quad C_{\gC(\widehat{\pi};\mu),\gC(\widehat{\pi}_{\delta}^{*};\mu)}\left(R_{\max}\sqrt{\gP\gL\log(\gP)}\frac{\big(\log|\gD|\big)^{\frac{2+\eta}{2\eta}}}{\sqrt{|\gD|}}\right)+C_{\gC(\widehat{\pi};\mu),\gC(\widehat{\pi}_{\delta}^{*};\mu)}\sqrt{\epsilon},
\end{align*}
where we employ the conclusion from Lemma \ref{lem:low_complex_approx}, $d_*=\max_i d_i$ and $\zeta_k^*=\min_i\zeta_i\prod_{l=i+1}^{k}(\zeta^{l}\land 1)$. In addition, we define $\zeta^*:=\zeta_k^*(1\land \zeta)$. Similarly, we set $N=\gO(|\gD|^{\frac{d_*}{2d_*+4\zeta^*}})$ and $M=\gO(\log(|\gD|))$. Consequently, we have $\gW=\gO(d_*^{s+1}|\gD|^{\frac{d_*}{2d_*+4\zeta^*}}),\gL=\gO(\log(|\gD|)),\gP=\gO(d_*^{2s+2}|\gD|^{\frac{d_*}{d_*+2\zeta^*}}\log(|\gD|))$. Rewriting the expression, we obtain
\begin{align*}
&\quad\Expect[\widetilde{\gR}_{\mu}(\widehat{\pi},\epsilon)-\widetilde{\gR}_{\mu}(\widehat{\pi}^{*},\epsilon)]  \\
&\leq C_{B,\zeta,s,k,\gC(\widehat{\pi}_{\delta}^{*};\mu)}d_*^{s+(\zeta\lor 1)/2} |\gD|^{\frac{-\zeta^*}{d_*+2\zeta^*}}\log(|\gD|) \\
&\quad +C_{\gC(\widehat{\pi};\mu),\gC(\widehat{\pi}_{\delta}^{*};\mu)}R_{\max}d_*^{s+1}|\gD|^{\frac{-\zeta^*}{d_*+2\zeta^*}}\log(|\gD|)^{2+\frac{1}{\eta}}+C_{\gC(\widehat{\pi};\mu),\gC(\widehat{\pi}_{\delta}^{*};\mu)}\sqrt{\epsilon} \\
&\leq C_{B,\zeta,s,k,\gC(\widehat{\pi};\mu),\gC(\widehat{\pi}_{\delta}^{*};\mu)}R_{\max}d_*^{s+(\zeta\lor 1)/2}|\gD|^{\frac{-\zeta^*}{d_*+2\zeta^*}}\log(|\gD|)^{2+\frac{1}{\eta}}+C_{\gC(\widehat{\pi};\mu),\gC(\widehat{\pi}_{\delta}^{*};\mu)}\sqrt{\epsilon} \\
&=C_{1}R_{\max}d_*^{s+(\zeta\lor 1)/2}|\gD|^{\frac{-\zeta^*}{d_*+2\zeta^*}}\log(|\gD|)^{2+\frac{1}{\eta}}+C_{2}\sqrt{\epsilon},
\end{align*}
where $\zeta^*=\zeta_k^*(1\land\zeta)=\min_i(\zeta_i\prod_{l=i+1}^{k}(\zeta^{l}\land 1))(1\land\zeta)$, $C_{1}$ is a constant that depends on $B,\zeta,s,k,\gC(\widehat{\pi};\mu),\gC(\widehat{\pi}_{\delta}^{*};\mu)$ and $C_{2}$ is a constant that depends on $\gC(\widehat{\pi};\mu),\gC(\widehat{\pi}_{\delta}^{*};\mu)$.

\end{document}